\documentclass{article}
\usepackage[dvipsnames]{xcolor}
\usepackage{kr}

\usepackage{amsthm}
\usepackage{thm-restate}
\usepackage{times}
\usepackage{soul}
\usepackage{url}
\usepackage[hidelinks]{hyperref}
\usepackage[utf8]{inputenc}
\usepackage[small]{caption}
\usepackage{graphicx}
\usepackage{amsmath}
\usepackage{booktabs}
\usepackage{algorithm}
\usepackage{algorithmic}
\usepackage{enumitem}

\usepackage{thm-restate}
\usepackage{etoolbox}

\usepackage{todonotes}
\newcommand{\TODO}[1]{}
\usepackage{amssymb}
\usepackage{complexity}

\newcommand{\ie}{i.e.}
\newcommand{\eg}{e.g.}
\newcommand{\suth}{s.t.\, }

\newcommand{\tuple}[1]{\ensuremath{\left(#1\right)}}

\usepackage{tikz}

\usetikzlibrary{shapes,decorations,shadows,arrows,calc}
\makeatletter
\tikzstyle{arg}=[draw,circle,fill=gray!15,inner sep=1pt,minimum size=.5cm]
\tikzstyle{argd}=[draw,circle,gray!50,inner sep=1pt,minimum size=.5cm]

\definecolor{darkgreen}{rgb}{0,0.4,0}
\tikzstyle{argTD}=[draw, thick, circle, fill=gray!15,inner sep=0pt,minimum size=0.6cm,font=\small]
\tikzstyle{argR}=[draw, thick, circle, fill=gray!15,inner sep=0pt,minimum size=0.45cm,font=\small]
\tikzstyle{scc}=[draw, thick, rectangle,align=center, fill=gray!15,inner sep=0pt,minimum size=0.8cm,font=\small, rounded corners=0ex,]
\tikzstyle{argTDX}=[draw, dotted,thick, circle, inner sep=0pt,minimum size=0.6cm,font=\small]
\tikzstyle{sccX}=[draw,dotted, thick, rectangle,align=center, inner sep=0pt,minimum size=0.8cm,font=\small, rounded corners=0ex,]
\tikzstyle{argsmall}=[draw, thick, circle, fill=gray!15,inner sep=0pt,minimum size=0.4cm]
\tikzstyle{argsmallX}=[draw, thick, circle, inner sep=0pt,minimum size=0.3cm,dotted]
\tikzstyle{bag}=[draw,rectangle, rounded corners=1ex,
align=center, inner sep=1ex]

\tikzstyle{atts}=[draw,thick, inner sep=5pt, rounded corners=3pt]

\newcommand{\pL}{{\bf (L)}}
\newcommand{\pF}{{\bf (F)}}
\newcommand{\pWF}{{\bf (WF)}}
\newcommand{\pP}{{\bf (PA)}}
\newcommand{\pPR}{{\bf (PRS)}}
\newcommand{\pSR}{{\bf (SR)}}
\newcommand{\pI}{{\bf (I)}}
\newcommand{\pM}{{\bf (M)}}
\newcommand{\pNE}{{\bf (NE)}}
\newcommand{\pMC}{{\bf (MC)}}
\newcommand{\pSTR}{\bf (S)}

\newcommand{\pUM}{\bf{(UM)}}
\newcommand{\pCC}{\bf{(CC)}}
\newcommand{\pURM}{\bf{(URM)}}

\newcommand{\cf}{\textit{cf}\,}
\newcommand{\adm}{\textit{adm}}
\newcommand{\com}{\textit{com}}
\newcommand{\comp}{\textit{com}}

\newcommand{\prf}{\textit{pref}}
\newcommand{\pref}{\textit{pref}}

\newcommand{\grd}{\textit{grd}}

\newcommand{\wadm}{\textit{adm}^{\textit{w}}}
\newcommand{\wpref}{\textit{pref}^{\textit{w}}}
\newcommand{\wprf}{\textit{pref}^{\textit{w}}}
\newcommand{\wcom}{\textit{com}^{\textit{w}}}
\newcommand{\wcomp}{\textit{com}^{\textit{w}}}
\newcommand{\wgrd}{\textit{grd}^{\textit{w}}}

\newcommand{\sadm}{\textit{adm}^{\textit{s}}}
\newcommand{\sad}{\textit{adm}^{\textit{s}}}
\newcommand{\sd}{\textit{sd}}
\newcommand{\dsd}{\ensuremath{\textit{sd}_{\Delta}}}

\newcommand{\spref}{\textit{pref}^{\textit{s}}}
\newcommand{\sprf}{\textit{pref}^{\textit{s}}}
\newcommand{\scom}{\textit{com}^{\textit{s}}}
\newcommand{\scomp}{\textit{com}^{\textit{s}}}
\newcommand{\sgrd}{\textit{grd}^{\textit{s}}}

\newcommand{\BF}{\mathit{F}}

\newcommand{\sdcomp}{\comp_\Delta^{\textit{s}}}

\newcommand{\sdcom}{\comp_\Delta^{\textit{s}}}

\newcommand{\sdgrd}{\grd_\Delta^{\textit{s}}}
\newcommand{\dsadm}{\adm_\Delta^{\textit{s}}}
\newcommand{\sdadm}{\adm_\Delta^{\textit{s}}}
\newcommand{\dspref}{\pref_\Delta^{\textit{s}}}
\newcommand{\sdpref}{\pref_\Delta^{\textit{s}}}

\newcommand{\gcom}{\comp_\Gamma}

\newcommand{\gadm}{\adm_\Gamma}

\newcommand{\wgadm}{\adm_\Gamma^{\textit{w}}}

\newcommand{\sgcom}{\comp_\Gamma^{\textit{s}}}

\newcommand{\sggrd}{\grd_\Gamma^{\textit{s}}}

\newcommand{\sgadm}{\adm_\Gamma^{\textit{s}}}

\newcommand{\sgpref}{\pref_\Gamma^{\textit{s}}}

\newcommand{\stadm}{\textit{adm}^{\textit{s}}}

\newcommand{\la}{\leftarrow}

\newcommand{\cl}{\mathit{cl}}
\newcommand{\asms}{\asm}
\newcommand{\theory}{\mathit{Th}}
\renewcommand{\D}{\ensuremath{\mathcal{D}}}
\newcommand{\lit}{\ensuremath{\mathcal{L}}}
\newcommand{\rules}{\ensuremath{\mathcal{R}}}
\newcommand{\asm}{\ensuremath{\mathcal{A}}}

\newcommand{\contrary}[1]{c(#1)}
\newcommand{\contraryc}[1]{{#1}_c}
\newcommand{\contraryempty}{c}

\newtheorem{example}{Example}[section]
\newtheorem{theorem}[example]{Theorem}
\newtheorem{definition}[example]{Definition}
\newtheorem{proposition}[example]{Proposition}

\newtheorem{lemma}[example]{Lemma}

\newenvironment{example*}[2][]
	{\pagebreak[2] \par \noindent \textbf{Example~\ref{#2} (continued).}\it}{\par}

\title{On Strong and Weak Admissibility in Non-Flat Assumption-Based Argumentation}

\author{%
Matti Berthold$^1$\and
Lydia Bl\"umel$^2$\and
Anna Rapberger$^3$\\
\affiliations
$^1$ScaDS.AI, Universit\"at Leipzig\\
$^2$Artificial Intelligence Group, University of Hagen\\
$^3$Imperial College London
\emails
berthold@informatik.uni-leipzig.de,
lydia.bluemel@fernuni-hagen.de,
a.rapberger@imperial.ac.uk
}

\begin{document}

\maketitle

\begin{abstract}
  In this work, we broaden the investigation of admissibility notions in the context of assumption-based argumentation (ABA).
  More specifically, we study two prominent alternatives to the standard notion of admissibility from abstract argumentation, namely strong and weak admissibility, and introduce the respective preferred, complete and grounded semantics for general (sometimes called non-flat) ABA.
  To do so, we use abstract bipolar set-based argumentation frameworks (BSAFs) as formal playground since they concisely capture the relations between assumptions and are expressive enough to represent general non-flat ABA frameworks, as recently shown.
  While weak admissibility has been recently investigated for a restricted fragment of ABA in which assumptions cannot be derived (flat ABA), strong admissibility has not been investigated for ABA so far. 
  We introduce \emph{strong admissibility} for ABA and investigate desirable properties. 
  We furthermore extend the recent investigations of \emph{weak admissibility} in the flat ABA fragment to the non-flat case. 
  We show that the central modularization property is maintained under classical, strong, and weak admissibility. We also show that strong and weakly admissible semantics in non-flat ABA share some of the shortcomings of standard admissible semantics and discuss ways to address these. 
\end{abstract}

\section{Introduction}
Computational argumentation is a dynamic and widely studied area within knowledge representation and reasoning~\cite{Gabbay:2021}. 
It provides foundational models for human reasoning processes when challenged with conflicting information
with the goal of identifying sets of jointly acceptable assumptions, premises, or arguments, representing  coherent and defensible viewpoints. 
A well-established formalism in this domain is assumption-based argumentation (ABA)~\cite{BondarenkoDKT97,CyrasFST2018},
which has been thoroughly investigated~\cite{HeyninckA24,LehtonenWJ21b,DBLP:conf/kr/LehtonenR0W23,DBLP:conf/kr/Rapberger024} and has applications in fields such as decision making~\cite{DBLP:conf/atal/FanTMW14,DBLP:journals/argcom/CyrasOKT21}, planning~\cite{DBLP:conf/prima/Fan18a}, and explainable AI~\cite{DBLP:conf/kr/0002RT24,DBLP:conf/kr/LeofanteADFGJP024}. 
In ABA frameworks (ABAFs), the reasoning process revolves around identifying acceptable sets of assumptions, which are the defeasible elements of the framework. 
The focus on inference rules and assumptions places ABA within the subfield of \emph{structured argumentation}, which explicitly accounts for the internal structure of arguments. In contrast, \emph{abstract argumentation}~\cite{Dung95} models argumentative processes by directed graphs where nodes are abstract arguments and edges represent attacks between them. 
Various argumentation semantics, largely shared across all major argumentation formalisms, are used to formalize when assumption sets can be considered acceptable, we refer to, \eg,  \cite{BaroniCG18,CyrasFST2018,DBLP:journals/argcom/BesnardGHMPST14} for an overview.

At the heart of most of these semantics lies the concept of \emph{admissibility}~\cite{Dung95}.
This notion, which can be seen as the argumentative counterpart to self-defense, is crucial for defining acceptance and defeat in computational argumentation.
\citeauthor{Dung95}~(\citeyear{Dung95}) introduced admissibility for abstract argumentation frameworks (AFs), encapsulated with the informal slogan: \emph{the one who has the last word laughs best}. 
A set of arguments is said to be admissible if it ensures both internal coherence (no self-conflict) and self-defense (ability to counter attacks).
In the context of ABA, an assumption $a$ can be challenged (\emph{attacked}) by a set of assumptions $S$ if $S$ derives the negation (the so-called \emph{contrary}) of~$a$.
A set of assumptions is then said to be \emph{admissible} if it is conflict-free and counters all of its attackers.

Beyond the classical notion, several refinements of admissibility have been proposed. 
\emph{Strong admissibility} strengthens defense: a member of a strongly admissible set cannot defend itself.
This property has been introduced by \citeauthor{BaroniG07}~(\citeyear{BaroniG07}) 
and first studied as an independent semantics by \citeauthor{DBLP:conf/comma/Caminada14}~(\citeyear{DBLP:conf/comma/Caminada14}) in the context of abstract argumentation;
using the following recursive definition:
a set $S$ is strongly admissible if each $a\in S$ is strongly defended by a subset of $S\!\setminus\! \{a\}$. It addresses concerns about circular justifications which are tolerated under standard admissibility. 
Strong admissiblity is well understood for AFs~\cite{DBLP:conf/comma/CaminadaD20,DBLP:conf/comma/CaminadaH24}
and has been studied in generalizations of AFs~\cite{DBLP:journals/argcom/ZafarghandiVV22,DBLP:journals/logcom/BistarelliT22}. 
In contrast, the realm of structured argumentation lacks similar investigations so far. 

On the other end of the spectrum, \emph{weak admissibility} has been proposed as a relaxation of the classical approach.
Originally introduced for AFs by~\citeauthor{BaumannBU2020}~(\citeyear{BaumannBU2020}), 
weak admissiblity addresses situations involving paradoxical arguments, 
such as self-attacking cycles, in which the standard notion of admissibility can be overly restrictive or unintuitive~\cite{Dung95,DondioL19}. 
The weaker version of the semantics allows for a more relaxed approach to defense that can ignore such paradoxes, focusing instead on maintaining the acceptability of reasonable arguments.
Crucial for the notion is the \emph{reduct} which removes arguments whose acceptance status is already decided, \ie, either accepted or defeated. The reduct is the technical prerequisite for the \emph{modularization property} which intuitively captures the idea that the evaluation of a semantics can be broken down into the evaluation on subframeworks, based on the reduct computation. 
In the context of AFs, weak admissibility is well understood~\cite{DBLP:journals/ai/BaumannBU22,DBLP:conf/clar/DauphinRT21,BluemelUlbricht22}.
In recent work, \citeauthor{BlumelKU24}~\shortcite{BlumelKU24} investigate weak admissibility in 
the context of ABA.
They focus, however, exclusively on the \emph{flat} ABA fragment in which assumptions cannot be derived~\cite{BondarenkoDKT97}. The flat setting restricts the framework’s expressiveness when applied to complex domains. In many contexts ---such as legal reasoning, ethical deliberation, or cognitive modeling--- assumptions are often interdependent, something flat ABA cannot represent. Despite its potential, the non-flat setting remains underexplored and offers a promising direction for extending the applicability and depth of assumption-based frameworks, especially in combination with alternative notions of admissibility.\\ 
Therefore, unlike prior work, we do not restrict ourselves to the flat case but explore admissiblity under these more general conditions. Naturally, this poses additional challenges as (i) in addition to attacks, assumptions can now also \emph{support} each other---analogously to the case of attacks, a set of assumptions $S$ supports an assumption $a$ if $S$ derives $a$---and (ii) assumption sets can carry implicit conflicts due to the assumptions they additionally derive. 
To address (ii), acceptability in non-flat ABAFs requires in addition to conflict-freeness and defense that the assumption set is \emph{closed}, that is, it contains all assumptions it derives.
This ensures that implicit conflicts are exposed and that the acceptability of a set can be meaningfully evaluated only when all its logical consequences are taken into account.

We base our investigations on \emph{bipolar set-based AFs (BSAFs)}~\cite{BertholdRU24} as they concisely
capture the attack and support relations between assumptions and abstract away from all elements such as rules, ordinary atoms, argument trees that are only indirectly used in the evaluation of an ABAF.

Overall, our main contributions are as follows.
\begin{itemize}
\item We introduce \emph{strong admissibility} for general (non-flat) ABA. 
We investigate fundamental properties of strong admissibility 
such as $\pUM$ which states that the $\subseteq$-maximal strongly admissible set is unique.
While some of these properties may be violated in the general case, 
the semantics behave as expected in the flat ABA fragment, as we show. 
\hfill {\color{gray}Section~\ref{sec:strong}}
\item We generalize \emph{weak admissibility} to non-flat ABA and study its properties. 
We show that some of the fundamental principles of weakly admissible semantics do not hold for non-flat ABA, \eg, removing self-attacking assumptions $\pP$ is not always possible. 
\hfill {\color{gray}Section~\ref{sec:weak}}
\item 
Towards generalizing weak admissiblity to the non-flat case, we define and investigate 
the BSAF reduct and the closely related modularization property, and show that modularization holds for both weak and strong admissiblity, as well as the classical semantics in non-flat ABA.
\hfill {\color{gray}Section~\ref{sec:reduct}} 
\item Inspired by recent attempts to alleviate the undesired behavior of ABA semantics in the non-flat case, we investigate the notion of $\Gamma$-closure in the context of weak and strong admissiblity. 
Intuitively, the $\Gamma$-closure of a set $S$ allows to ignore assumptions that are not defended by $S$. 
While the adjustment fails to fix weak semantics, 
we show that the resulting \emph{strongly $\Gamma$-admissible} semantics retains all desired properties. 
\hfill {\color{gray}Section~\ref{sec:FL-sensitive semantics}}
\end{itemize}
All proofs are provided in the supplementary material.

\section{Background}
\label{sec:background}
\subsection{Assumption-based Argumentation}\label{sec:ABA}
We recall the technical definitions
of (ABA)~\cite{CyrasFST2018}. We assume a \emph{deductive system}, \ie\, a tuple $(\lit,\rules)$, where  
$\lit$ is a set of atoms and 
$\rules$ is a set of inference rules over $\lit$. 
A rule $r \in \rules$ has the form
$a_0 \leftarrow a_1,\ldots,a_n$, \suth $a_i \in \lit$ for all $0\leq i\leq n$; $head(r) := a_0$ is the \emph{head} and $body(r) := \{a_1,\ldots,a_n\}$ is the (possibly empty) \emph{body} of $r$.
\begin{definition}
	An ABA framework (ABAF) is a tuple $(\lit,\rules,\asm,\contraryempty)$, where $(\lit,\rules)$ is a deductive system, $\asm \subseteq \lit$ a set of assumptions, and  
	$\contraryempty:\asm\rightarrow \lit$ a contrary function.
\end{definition}
We fix an arbitrary ABAF $\D=(\lit,\rules,\asm,\contraryempty)$ below.
The ABAF $\D$ is \emph{flat} iff $head(r)\notin \asm$ for all $r\in \rules$.
In this work we focus on \emph{finite} ABAFs, \ie\, $\lit$ and $\rules$ are finite. 

An atom $p \in \lit$ is tree-derivable from assumptions $S \subseteq \asm$ and rules $R \subseteq \rules$, denoted by $S \vdash_R p$, if there is a finite rooted labeled tree $t$ s.t.\
i) the root of $t$ is labeled with $p$, 
ii) the set of labels for the leaves of $t$ is equal to $S$ or $S \cup \{\top\}$, and 
iii) for each node $v$ that is not a leaf of $t$ there is a rule $r\in R$ such that $v$ is labeled with $head(r)$ and labels of the children correspond to $body(r)$ or $\top$ if $body(r)=\emptyset$.
We write $S\vdash p$ iff there exists  $R \subseteq \rules$ such that $S\vdash_R p$. 

Let $S\subseteq \asm$.
By $\contrary{S}:=\{\contrary{a}\mid a\in S\}$  we denote the set of all contraries of $S$. 
Set $S$ \emph{attacks} a set $T\subseteq\asm$ if there are $S'\subseteq S$ and $a\in T$ s.t.\ $S' \vdash \contrary{a}$; if $S$ attacks $\{a\}$ we say $S$ attacks~$a$. 
$S$ is conflict-free ($S\in\cf(\D)$) if it does not attack itself. 
The closure $cl(S)$ of $S$ is $cl(S):= \theory_\D(S)\cap \asm$ where $\theory_\D(S):=\{p\in L\mid\exists S'\subseteq S:S' \vdash
p\}$ denotes all derived conclusions.
We write $cl(a)$ instead of $cl(\{a\})$ for singletons. 
We call $S\subseteq \asm$ \emph{closed} if $S = cl(S)$. 

Now we consider defense \cite{BondarenkoDKT97,CyrasFST2018}. Observe that defense in general ABAFs is only required against closed sets of attackers. 

\begin{definition}
A set $S$ of assumptions defends an assumption $a$ iff for each closed set $T$ 
which attacks $a$, we have $S$ attacks $T$; $S$ defends itself iff $S$ defends each $b\in S$.
\end{definition}
A set $E$ of assumptions is \emph{admissible} ($E\in\adm(\BF)$) iff $E$ is conflict-free, closed and defends itself.
We next recall grounded, complete, and preferred ABA semantics.

\begin{definition}\label{def:ABA semantics}
	Let $\D$ be an ABAF and $S\in\adm(\D)$. Then
	\begin{itemize}
		\item 
		$S\in\comp(\D)$ iff it contains every assumption it defends; 
		\item 
		$S\in\grd(\D)$ iff $S$ is $\subseteq$-minimal in $\comp(\D)$;
		\item 
		$S\in\pref(\D)$ iff $S$ is $\subseteq$-maximal in $\adm(\D)$.
	\end{itemize}
\end{definition}
We denote by $\Sigma = \{\adm,\com,\grd,\prf\}$ the family of (classical, admissible-based) Dung semantics.

To study restrictions of ABAFs, the following notation will be useful.
For an ABAF $\D=\tuple{\lit,\rules,\asm,\contraryempty}$ and $S\subseteq \mathcal{A}$, we write $\D\downarrow_{S}$ for the framework that arises when restricting $\D$ to $S$, \ie, $\D\downarrow_{S}:=(\lit,\rules,S,\contraryempty|_{S})$.

\subsection{Bipolar Set-based Abstract Argumentation}
Bipolar SETAFs~\cite{BertholdRU24} combine the ideas underlying argumentation frameworks with collective attacks (SETAFs)~\cite{NielsenP06} and bipolar argumentation frameworks (BAFs)~\cite{CayrolL05a,Amgoud08,PoRaUlAAAI2024}. Instead of only considering an attack relation, there is also a notion of support. Bipolar SETAFs (BSAFs) can model \emph{collective} attacks and supports. 

\begin{definition}
    A bipolar set-argumentation framework (BSAF) is a tuple $\BF = \tuple{A,R,S}$, where $A$ is a finite set of arguments, $R \subseteq 2^A\times A$ is the attack relation and $S\subseteq 2^A\times A$ is the support relation.
\end{definition}
In this work, we consider finite BSAFs only, \ie, $A$ is finite.
A SETAF is a BSAF  $\BF = \tuple{A,R,S}$ with $S=\emptyset$; 
an AF is a SETAF with $|T|=1$ for all $(T,h)\in R$.
\begin{definition}
	Given BSAF $\BF\! = \! (A,R,S)$ and $E\subseteq A$, let
		$$supp_{\BF}(E):=E\cup \{h\in A\mid \exists\tuple{T,h}\in S: T\subseteq E\}.$$ 
	We define the \emph{closure} $\cl_{\BF}(E):=\bigcup_{i\geq 1} supp_{\BF}^i(E)$ of $E$; 
	$E$ is \emph{closed} if $cl_{\BF}(E)=E$. 
\end{definition}

\begin{definition}
\label{def:BSAF gamma and more}
Given BSAF $\BF\! = \! (A,R,S)$, a set $E\!\subseteq\! A$ 
\emph{defends} $a\!\in\! A$ if for each closed attacker $E'\!\subseteq\! A$ of~$a$, $E$ attacks $E'$;
$E$ \emph{defends} $E'$ if $E$ defends each $a\!\in\! E'$.
The \emph{characteristic function} is $\Gamma_{\BF}(E):=\{a\in A\mid E\text{ defends }a\text{ in }\BF\}$. 
\end{definition}
We omit the subscript $\BF$ for $\Gamma$ and $\cl$ if clear from context.

Let us now head to BSAF semantics. 
A set $E$ is conflict-free ($E\!\in\!\cf(\BF)$) if it does not attack itself;
$E$ is admissible ($E\in\adm(\BF)$) if it is conflict-free, closed and defends itself.%
\begin{definition}\label{def:BSAF semantics}
	Let $\BF$ be an BSAF and let $E\in\adm(\BF)$. 
	\begin{itemize}
		\item 
		$E\in \comp(\BF)$ iff $E$ contains every assumption it defends; 
		\item 
		$E\in \grd(\BF)$ iff $E$ is $\subseteq$-minimal in $\comp(\BF)$;
		\item 
		$E\in \pref(\BF)$ iff $E$ is $\subseteq$-maximal in $\adm(\BF)$.
	\end{itemize}
\end{definition}

Given a BSAF $(A,R,S)$ and a set of arguments $E\subseteq A$, we denote $E^+_R:=\{h\mid \exists T\subseteq A: (T,h)\in R\}$ and the range of $E$ by $E^\oplus_R:= E_R\cup E^+_R$. The index $R$ may be omitted, if clear from the context.
Graphically, we depict the attack relation of a BSAF by solid edges and the support relation by dashed edges (cf.\ Example~\ref{ex:background}). 

\newcommand\distance{1.5}

\subsubsection{ABA and BSAF}
BSAFs capture non-flat ABAFs, as shown by~\citeauthor{BertholdRU24}~(\citeyear{BertholdRU24}).
\begin{definition} 
	\label{def:non-flat collective inst}
	Let $\D=\tuple{\lit,\rules,\asm,\contraryempty}$ be an ABAF. 
	Then we set $\BF_\D:=\tuple{A,R,S}$, where $A := \asm$ and 
	\begin{align*}
		R:={}&\{(T,h)\mid h\in\asm,\; T\subseteq\asm,\; T\vdash \contrary{h} \},\\
		S:={}&\{(T,h)\mid h\in\asm,\; T\subseteq\asm,\; T\vdash h\}\}
	\end{align*}
\end{definition}

\citeauthor{BertholdRU24}~\shortcite{BertholdRU24} show that the BSAF abstraction of an ABAF preserves the semantics.  
\begin{theorem}
	\label{prop:BSAF2ABAFCorrect}
	$\sigma(\D)\!=\!\sigma(\BF_\D)$ 
	for any $\sigma\in\Sigma$ and ABAF~$\D$.
\end{theorem}

\begin{example}\label{ex:background}
We consider an ABAF 
$\D=\tuple{\lit,\rules,\asm,\contraryempty}$ with
literals $\allowbreak\lit = \{ a,b,c,d,e,\contraryc{a},\contraryc{b},\contraryc{c},\contraryc{d},\contraryc{e}\}$, 
assumptions $\asm = \{a,b,c,d,e\}$, their contraries $\contraryc{a}$, $\contraryc{b}$, $\contraryc{c}$, $\contraryc{d}$, $\contraryc{e}$, respectively, and rules
\begin{align*}
\contraryc{c}\la d &&\contraryc{e}\la e && e\la a,b && \contraryc{d}\la c &&\contraryc{e}\la	 b,c&&
\end{align*}
We compute the BSAF $\BF_\D=\tuple{A,R,S}$:
the nodes correspond to the assumptions, \ie, $A=\{a,b,c,d,e\}$,
the attacks to the tree-derivations that derive contraries, \ie, $R=\{$\textcolor{black}{$\tuple{\{d\},c}$}, \textcolor{black}{$\tuple{\{c\},d}$}, $\textcolor{Orange}{\tuple{\{b,c\},e}},$ $\textcolor{black}{\tuple{\{e\},e}}\}$, and the supports correspond to the tree-derivations that derive assumptions, in this case $S=\{\textcolor{cyan}{\tuple{\{a,b\},e}}\}$. 
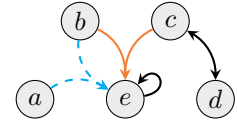
\vspace{-.4cm}
\begin{center}
		\begin{tikzpicture}[scale=0.8,>=stealth]
		\path
        (\distance-1.1,.85)node (D){}
        (0,0)node[arg] (b){$b$}
		(\distance,0)node[arg] (c){$c$}
		(.5*\distance,-\distance*0.866)node[arg] (e){$e$}
		(-.5*\distance,-\distance*0.866)node[arg] (a){$a$}
  		(1.5*\distance,-\distance*0.866)node[arg] (d){$d$}
        ;
        
\path[thick,<->]
(c)edge[bend left](d)
;

\path[thick,->,Orange]
(b)edge[out=-25,in=90](e)
(c)edge[out=-155,in=90](e)
;

\path[thick,->,cyan,dashed]
(a)edge[out=35,in=150](e)
(b)edge[out=-95,in=150](e)
;
\path[thick,->]
(e)edge[in=50,out=10,min distance=5mm](e)
;
		
		\end{tikzpicture}
	\end{center}
The BSAF provides an easy-to-understand graphical representation to evaluate the semantics in $\D$.
The sets $\{a\}$ and $\{b\}$ are admissible since they are unattacked; however, they cannot be accepted together because they jointly support the self-attacker $e$. Thus, the BSAF (and therefore the ABAF) has neither complete nor grounded extensions. It has three preferred extensions, $\{a,c\}$, $\{b,c\}$ and $\{b,d\}$.
\end{example}

\section{Strong and Weak Admissibility in ABA}
\label{sec:desiderata}

Our goal is to study semantics based on strong and weak admissibility in ABA in a principled way.
We will thus identify desirable properties that we expect our new families of semantics adhere to.

To get a better understanding of the oftentimes involved technicalities of the recursive definitions of strong and weak admissibility, we make use of the close relationship between BSAFs and ABAFs.
BSAFs concisely capture the attack and support relations between assumptions
while abstracting away from all components of an ABAF that are only implicitly needed to compute the extensions, as discussed in Example~\ref{ex:background}.
We will therefore utilize BSAFs as our formal playground to rigorously define and investigate our novel semantics. 
 
In the remainder of this section, we consider an arbitrary but fixed ABAF $\D=\tuple{\lit,\rules,\asm,\contraryempty}$.

\paragraph{Global Desiderata}\label{subsec:global des}

We discuss desired properties for both strong and weak admissiblility.
A central property of argumentation semantics is the fundamental lemma, introduced by Dung~\shortcite{Dung95}. It states that each assumption defended by an admissible set $E$ can be accepted together with $E$. 
\begin{itemize}[align=left]
	\item[$\pF$] Fundamental Lemma: If $S\in\adm(\D)$ defends $a$, then $S\cup\{a\}\in\adm(\D)$.
\end{itemize}
We recall basic semantics relations; also, all semantics 
 are expected to return a result.
 
\begin{itemize}[align=left]
    \item[$\pSR$] Semantics Relations: $\pref(\D)\subseteq \com(\D)\subseteq\adm(\D)$ and $\grd(\D)\subseteq \com(\D)$
    \item[$\pNE$] Non-empty: $\adm(\D), \com(\D), \pref(\D), \grd(\D)\! \neq\!\emptyset$
\end{itemize}
We note that $\pSR$ implies the \emph{maximal complete principle} $\pMC$~\cite{BlumelKU24} which states that the preferred extensions correspond to the maximal complete extensions of a given framework.
\begin{restatable}{proposition}{PropMCiffSR}
    \label{prop:MCiffSR}
    $\pSR\ \Rightarrow\ \pMC$ where $\pMC$ denotes $\pref(\D) = \{E\subseteq A\mid E\text{ is maximal in }\com(\D)\}$ 
\end{restatable}

Central to both strong and weak admissibility is the modularization property. 
A semantics satisfies modularization if its extensions can be computed iteratively by projecting away all elements that are already known to be either accepted or defeated. Crucial for this property is the reduct of a framework, originally introduced in~\cite{BaumannBU2020}.
Below, we recall the definition for SETAFs as they correspond to flat ABAFs~\cite{DvorakKUW21}.
\begin{definition}\label{def_setaf_reduct}
	Given a SETAF $F=(A,R)$ and $E\subseteq A$, the \emph{$E$-reduct of $F$} is the SETAF $F^E:=(A^E,R^E)$, with
	\begin{align*}
		A^E\enspace:=\enspace& A \setminus E^\oplus_R\\
		R^E\enspace:=\enspace& \{(T\setminus E, h) \mid (T,h)\in R, \,
		T\cap E^+_R=\emptyset,\\
		&\hphantom{\{} T \not\subseteq E,\, h\in A^E
		\}
	\end{align*}
\end{definition}

\newcommand\shortdistance{1.2}

\begin{example}
Consider the SETAF $F$ and its reduct $F^E$ wrt.\ $E=\{a\}$ (with removed arguments and attacks in light-gray), depicted below.   
    
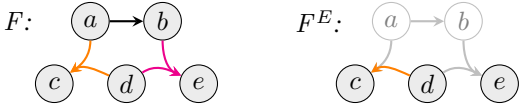
\begin{center}
    	\begin{tikzpicture}[scale=0.8,>=stealth]
        \begin{scope}[xshift=0cm]
    		\path
            (-0,0)node (D){$\BF$:}
            (\shortdistance,0)node[arg] (b){$a$}
            (2*\shortdistance,0)node[arg] (c){$b$}
            (.5*\shortdistance,-.866*\shortdistance)node[arg] (d){$c$}
            (1.5*\shortdistance,-.866*\shortdistance)node[arg] (e){$d$}
            (2.5*\shortdistance,-.866*\shortdistance)node[arg] (f){$e$}
            ;
            
            \path[thick,->]
            (b)edge(c)
            ;
            \path[thick,->,magenta]
            (c)edge[out=-90,in=155](f)
            (e)edge[out=35,in=155](f)
            ;
            \path[thick,->,orange]
            (b)edge[out=-90,in=35](d)
            (e)edge[out=155,in=35](d)
            ;
        \end{scope}
        
        \begin{scope}[xshift=5cm]
    		\path
            (0,0)node (D){$\BF^{E}$:}
            (\shortdistance,0)node[argd] (b){\color{gray}$a$}
            (2*\shortdistance,0)node[argd] (c){\color{gray}$b$}
            (.5*\shortdistance,-.866*\shortdistance)node[arg] (d){$c$}
            (1.5*\shortdistance,-.866*\shortdistance)node[arg] (e){$d$}
            (2.5*\shortdistance,-.866*\shortdistance)node[arg] (f){$e$}
            ;
            
            \path[thick,->,lightgray]
            (b)edge(c)
            ;
            \path[thick,->,lightgray]
            (c)edge[out=-90,in=155](f)
            (e)edge[out=35,in=155](f)
            ;
            \path[thick,->,lightgray]
            (b)edge[out=-90,in=35](d)
            ;
            \path[thick,->,orange]
            (e)edge[out=155,in=35](d)
            ;
        \end{scope}
        
        \end{tikzpicture}
    \end{center}
    The reduct $\BF^E$ of $\BF$ wrt. $E$ assumes $a$ to be true. Thus, $b$ is assumed to be false, since it is attacked by $E$. Both arguments $a$ and $b$ are therefore removed. 
The attack {\color{magenta}$(\{b,d\},e)$} is deactivated because $b$ is out, thus it is removed entirely;
the attack    {\color{orange}$(\{a,d\},c)$ } can still fire, since $a$ is in, thus, we adjust it accordingly  and keep {\color{orange}$(\{d\},c)$ } in $F^E$. 
\end{example}
We are ready to define \emph{modularization} which allows to compute extensions in a modular fashion.
The property holds for (SET)AFs~\cite{BBU2020Modularization,DvorakKUW24}, and we expect this property to hold for BSAFs (that resemble non-flat ABAFs) as well. 

\begin{itemize}[align=left]
    \item[$\pM$] Modularization: A semantics $\sigma$ satisfies modularization iff for each $E\subseteq A,E'\subseteq A^E$ we have
    $E\in\sigma( \BF)$ and $E'\in\sigma( \BF^E)$ implies $E\cup{E'}\in\sigma( \BF)$.
\end{itemize}

\paragraph{Strong Admissibility Desiderata} 
Strong admissibility strengthens the standard notion by requiring that each member of a strongly admissible set is defended by a strongly admissible subset that does not contain it. The notion has first been introduced in the scope of a principle-based analysis by \citeauthor{BaroniG07}~(\citeyear{BaroniG07}) for AFs and has first been studied as an independent semantics concept by \citeauthor{DBLP:conf/comma/Caminada14}~\shortcite{DBLP:conf/comma/Caminada14}. We recall the AF definition.
\begin{definition}
    \label{def:strong_adm_af}
    Let $F=(A,R)$ be an AF. A conflict-free set $E\in \cf(F)$ is strongly admissible in $F$ iff each $a\in E$ is defended by a strongly admissible set $E'\subseteq E\setminus\{a\}$. 
\end{definition}
Similar to grounded semantics, strongly admissible extensions can be computed starting from a set of undefeated arguments. 
We discuss the idea in the following example.
\begin{example}
    Consider the following AF.
    \begin{center}
    	\begin{tikzpicture}[scale=1,>=stealth]
    		\path
    		(0,0) node[arg] (a){$a$}
    		(1.2,0)node[arg] (b){$b$}
      		(2.4,0)node[arg] (c){$c$}
            (3.6,0)node[arg] (d){$d$}
            (4.8,0)node[arg] (e){$e$}
            ;
            
            \path[thick,->]
            (c)edge (b)
            (d) edge (e)
;
\path[thick,<->]
            (b)edge (a)
;
        \end{tikzpicture}
    \end{center}
A strongly admissible extension may, or may not contain $c$ or $d$, since they are unattacked.
Accordingly, an extension may contain $a$, if it contains $c$.
Note that $a$ is not contained in a strongly admissible extension if $c$ is not present.
We have $\sad(\BF)=\{\emptyset,$ $\{c\},$ $\{c,a\},$ $\{d\},$ $\{c,d\},$ $\{c,d,a\}\}$.
\end{example}

Analogous to the AF case, we expect that strong admissibility for ABA strengthens admissibility. 
\begin{itemize}[align=left]
    \item[$\pSTR$] Strengthening: It holds that $\sadm(\D)\subseteq\adm(\D)$.
\end{itemize}

We consider the following three desiderata, inspired by Caminada~(\citeyear{DBLP:conf/comma/Caminada14}).
Unique maximum states that the set of strongly admissible extensions has a unique maximal element; 
Unique relative maximum states that this property also holds relative to a given admissible set; and complete containment states that each strongly admissible set is contained in each complete extension.
\begin{itemize}[align=left]
    \item[$\pUM$] Unique Maximum: The $\subseteq$-maximal strongly admissible extension is unique. 
    \item[$\pURM$] Unique Relative Maximum: Each admissible set has a unique $\subseteq$-maximal strongly admissible subset.
    \item[$\pCC$] Complete Containment: $E\in\sad(\D)$ and $E'\in\com(\D)$ imply $E\subseteq E'$.
\end{itemize}
Note that for AFs, $\pCC$ implies that all strongly admissible sets are a subset of the (unique) grounded extension.

\paragraph{Weak Admissibility Desiderata}

\citeauthor{BlumelKU24}~\shortcite{BlumelKU24} introduced weak admissiblity for the flat ABA fragment, basing their definition on SETAFs. 
\begin{definition}\label{def_setaf_wadm}
		Let $ \BF=(A,R)$ be a SETAF, let $E\subseteq A$ be a set of arguments, and $ \BF^E=(A^E,R^E)$ its $E$-reduct. 
		Then $E$ is called \emph{weakly admissible} in $ \BF$ ($E\in \wadm( \BF)$) iff
	\begin{enumerate}
		\item $E\in \cf( \BF)$ and
		\item for each $(T,h)\in R$ with $h\in E$, and
		$T\cap E^+_R=\emptyset$ 
		it holds $\nexists E' \in \wadm( \BF^E)$ s.t.\ $T\cap A^E\subseteq E'$.
	\end{enumerate}
Let $\D=(\lit,\asm,\rules,\contraryempty)$ be an ABAF and $\BF_\D$ the corresponding SETAF (cf.\ Definition~\ref{def:non-flat collective inst}). A set $E\subseteq \asm$ of assumptions is \emph{weakly admissible} ($E\in \wadm(\D)$) iff $E\in\wadm(\BF_\D)$. 
\end{definition}
In their work, they identified 
desirable properties for weakly admissible semantics which we will recall below.
First, the semantics is expected to weaken the traditional admissibility notion.

\begin{itemize}[align=left]
    \item[$\pL$] Liberalization: It holds that $\adm(\D)\subseteq\wadm(\D)$.
\end{itemize}
A fundamental principle of weak admissibility is that it allows for deleting so-called paradoxical components. In AFs, these components are, for instance, self-attackers or odd cycles in general. In the case of ABAFs, \citeauthor{BlumelKU24}~(~\citeyear{BlumelKU24}) identified the following counterpart for paradoxical assumptions. 
\begin{itemize}[align=left]
    \item[$\pP$] Paradoxical Assumptions: If $\{a\}\vdash \contrary{a}$, then it holds
that $\wadm(\D) = \wadm(\D\downarrow_{A\setminus\{a\}})$.
\end{itemize}
Finally, we consider a novel principle similar to the paradoxical rule principle in~\cite{BlumelKU24}.
The principle involves attacks and supports in the ABAF that we deem paradoxical.
We introduce the concept (in terms of BSAFs as they resemble ABA attacks and supports) below.
\begin{definition}
    Given a BSAF $\BF=\tuple{A,R,S}$. An attack $r=(T,h)\in R$ is \emph{paradoxical} iff $T\neq\emptyset$ and for every $t\in T$ there is a $T'\subseteq T, T'\neq\emptyset$ s.t.
    \begin{itemize}
        \item there exists $(T',t)\in R$ and $h\notin T'$.
    \end{itemize}
    A support $s=(T,h)\in S$ is \emph{paradoxical} iff $T\neq\emptyset$ and for every $t\in T$ there is a $T'\subseteq T, T'\neq\emptyset$ s.t.
    \begin{itemize}
        \item there exists $(T',t)\in R$.
    \end{itemize}
\end{definition}

\begin{itemize}[align=left]
    \item[$\pPR$] Paradoxical Attacks/Supports: Removing a paradoxical attack $r$ or support $s$ does not alter the models of $\BF$, \ie\, $\wadm(\BF)=\wadm(\BF')$ where $\BF'= \tuple{A,R\setminus\{r\},S}$ (resp. $\BF'= \tuple{A,R,S\setminus\{s\}}$).
\end{itemize}

\paragraph{Section Outline}
In the following sections, we will define strong and weak admissibility and study them in terms of the desiderata that we identified, thereby utilizing BSAFs as a formal playground. 
First, we will introduce the \emph{BSAF reduct} and study modularization for classical Dung semantics in general (potentially non-flat) ABA (cf.\ Section~\ref{sec:reduct}). 
Second, we introduce \emph{strong admissibility}, first for BSAFs, and subsequently for ABA in Section~\ref{sec:strong} where we study its behavior for non-flat ABA and in the flat fragment.
Third, we discuss \emph{weak admissibilty} for non-flat ABA in Section~\ref{sec:weak}. 

The classical semantics general (potentially non-flat) ABA are known to admit undesired behavior, as, \eg, discussed in~\cite{HeyninckA24,BertholdRU24}.
We (correctly) anticipate similar issues with the generalizations of the semantics based on strong respectively weak admissibility. 
We tackle these issues in Section~\ref{sec:FL-sensitive semantics} and consider revised versions of the semantics to reinstate $\pF$,$\pSR$ and $\pNE$, as recently proposed in~\cite{BertholdRU24}. 

\section{The BSAF Reduct}
\label{sec:reduct}
Towards modularization and weak admissibility for BSAFs, we first introduce the \emph{$E$-reduct for BSAFs}. 
As for the AF and SETAF reduct~\cite{BaumannBU2020,DvorakKUW24}, the BSAF reduct should capture the intuition that the arguments in a given set $E$ are true, and
all arguments in $E^+$ are false. 

It is tempting to proceed analogous to the (SET)AF case and set all arguments that are attacked to false (and remove them). However, when doing so, we encounter some issues.

\begin{example}
    \label{ex:bsaf_reduct}
Consider the BSAF $\BF$ as depicted below (left) and
the result $F'$ of the SETAF reduct computation wrt.\ $E=\{d\}$ (cf.~Definition~\ref{def_setaf_reduct}) (right). 
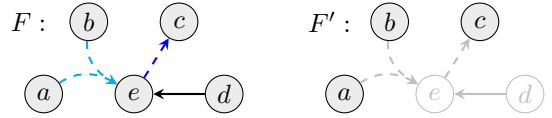
\vspace{-10pt}
    \begin{center}
    	\begin{tikzpicture}[scale=0.8,>=stealth]
    		\path
    		(-1,0) node {$F:$}
            (\distance-1.1,.85)node (D){}
            (0,0)node[arg] (b){$b$}
    		(\distance,0)node[arg] (c){$c$}
    		(.5*\distance,-\distance*0.8)node[arg] (e){$e$}
    		(-.5*\distance,-\distance*0.8)node[arg] (a){$a$}
      		(1.5*\distance,-\distance*0.8)node[arg] (d){$d$}
            ;
            
            \path[thick,->]
            (d)edge (e)
            ;
            \path[thick,->,dashed,blue]
            (e)edge (c)
            ;
            \path[thick,->,cyan,dashed]
            (a)edge[out=35,in=150](e)
            (b)edge[out=-95,in=150](e)
            ;
            
         \begin{scope}[xshift=5cm]
             		\path
    		(-1,0) node {$F':$}
            (\distance-1.1,.85)node (D){}
            (0,0)node[arg] (b){$b$}
    		(\distance,0)node[arg] (c){$c$}
    		(.5*\distance,-\distance*0.8)node[argd] (e){$e$}
    		(-.5*\distance,-\distance*0.8)node[arg] (a){$a$}
      		(1.5*\distance,-\distance*0.8)node[argd] (d){$d$}
            ;
            
            \path[thick,->,lightgray]
            (d)edge (e)
            ;
            \path[thick,->,dashed,lightgray]
            (e)edge (c)
            ;
            \path[thick,->,lightgray,dashed]
            (a)edge[out=35,in=150](e)
            (b)edge[out=-95,in=150](e)
            ;
         \end{scope}	
        \end{tikzpicture}
    \end{center}
If we proceed as in the SETAF case, 
the argument $e$ is defeated and thus can be removed.
The resulting BSAF thus contains no attacks and supports, thus all arguments can be accepted in the next iteration.
With supports present, however, we need to be a bit more careful since accepting the arguments $a$ and $b$ would require to accept the previously defeated argument $e$ as well.
Thus, we need to ensure that the set $a$ and $b$ cannot be jointly accepted. 
To do so, we will add attacks from the set $\{a,b\}$ to each of its members. 
\end{example}
As shown in the previous example, 
the removal of an attack has consequences for all its supporting sets. 
To prevent that all arguments in a set $S$ that supports an already defeated argument can get accepted, we add self-attacks $(S,a)$ to all $a\in S$; this acts as a constraint. 
As soon as one of the arguments in the set gets defeated the attacks will be removed and the remaining arguments can be accepted. 

We encounter another subtlety.
\begin{example}
    \label{ex:bsaf_reduct one}
    Consider the BSAF $\BF$ as depicted below (left) and
the SETAF reduct $F'$ wrt.\ $E=\{a\}$ (right). 
    \begin{center}
    	\begin{tikzpicture}[scale=0.8,>=stealth]
    		\path
    		(-1,0) node {$F:$}
    		(0,0) node[arg] (a) {$a$}
    		(1.2,0) node[arg] (b) {$b$}
    		(2.4,0) node[arg] (c) {$c$};

            \path[thick,->,cyan, dashed]
            (a)edge (b)
            ;
            \path[thick,->]
            (b)edge (c)
            ;

         \begin{scope}[xshift=5cm]
             		\path
    		(-1,0) node {$F':$}
    		(0,0) node[argd] (a) {\color{lightgray}$a$}
    		(1.2,0) node[arg] (b) {$b$}
    		(2.4,0) node[arg] (c) {$c$};

            \path[thick,->,lightgray, dashed]
            (a)edge (b)
            ;
            \path[thick,->]
            (b)edge (c)
            ;
         \end{scope}	
        \end{tikzpicture}
    \end{center}
    However, since $a$ supports $b$, we already know that $b$ can be deemed accepted as well; and, consequently, $c$ is defeated.
\end{example}
In the BSAF reduct, we will thus remove all true and defeated arguments wrt.\ the \emph{closure} of $E$.
 
We define the reduct as follows. 
\begin{itemize}
\item First, we add constraints, 
for each support $(T,h)$ with $h\in cl(E)^+_R$, we add attacks $(T,t)$ for all $t\in T$; 
\item Next, we remove all true and defeated arguments: compute the closure $\cl(E)$ of $E$, remove all arguments that are contained in $\cl(E)$ since they are true wrt.\ $E$, remove all arguments in $\cl(E)^+_R$ since they are defeated wrt.\ $E$;
\item  Now, we remove all attacks and supports $(T,h)$ whenever $h\in \cl(E)^\oplus$ or $T\cap \cl(E)^+_R\neq \emptyset$;
\item Finally, we adjust the attacks and supports by 
restricting the remaining attacks and supports to  $(T\setminus  \cl(E),h)$.
\end{itemize}

\begin{definition}\label{def_bsaf_reduct}
	Given a BSAF $\BF=(A,R,S)$ and $E\subseteq A$, the \emph{$E$-reduct of $\BF$} is the BSAF $\BF^E:=(A^E,R^E,S^E)$, with
	\begin{align*}
		A^E := &\ A \setminus (\cl(E)^\oplus_R)\\
		R^E := &\ \{(T\setminus \cl(E),t) \mid  \exists h\in \cl(E)^+_R: (T,h)\in S,\\
		& \phantom{\ \{(T\setminus \cl(E),t) \mid\quad}   t\in T\cap A^E\}\, \cup  \\
		&\ \{(T\setminus \cl(E), h) \mid  T\cap \cl(E)^+_R=\emptyset, \\
		& \phantom{\ \{(T\setminus \cl(E),t) \mid\quad} (T,h)\in R^E, h\in A^E
		\}\\
        S^E := 
		&\ \{(T\setminus \cl(E), h) \mid  T\cap \cl(E)^+_R=\emptyset, \\
		& \phantom{\ \{(T\setminus \cl(E),t) \mid\quad} (T,h)\in S, h\in A^E
		\}
	\end{align*}
\end{definition}

\begin{example*}{ex:bsaf_reduct}
    We depict the previous BSAF $\BF$ and its BSAF reduct $\BF^E$ for $E=\{d\}$ below.
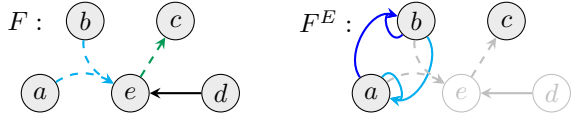
\vspace{-10pt}
    \begin{center}
    	\begin{tikzpicture}[scale=0.8,>=stealth]
    		\path
    		(-1,0) node {$F:$}
            (\distance-1.1,.85)node (D){}
            (0,0)node[arg] (b){$b$}
    		(\distance,0)node[arg] (c){$c$}
    		(.5*\distance,-\distance*0.8)node[arg] (e){$e$}
    		(-.5*\distance,-\distance*0.8)node[arg] (a){$a$}
      		(1.5*\distance,-\distance*0.8)node[arg] (d){$d$}
            ;

            \path[thick,->]
            (d)edge (e)
            ;
            \path[thick,->,dashed,ForestGreen]
            (e)edge (c)
            ;
            \path[thick,->,cyan,dashed]
            (a)edge[out=35,in=150](e)
            (b)edge[out=-95,in=150](e)
            ;
            
         \begin{scope}[xshift=5.5cm]
             		\path
    		(-1.5,0) node {$F^E:$}
            (\distance-1.1,.85)node (D){}
            (0,0)node[arg] (b){$b$}
    		(\distance,0)node[arg] (c){$c$}
    		(.5*\distance,-\distance*0.8)node[argd] (e){$e$}
    		(-.5*\distance,-\distance*0.8)node[arg] (a){$a$}
      		(1.5*\distance,-\distance*0.8)node[argd] (d){$d$}
            ;
            
            \path[thick,->,blue]
            (a)edge[out=125,in=165](b)
            (b)edge[out=-135,looseness=2.5,in=165](b)
            ;
            \path[thick,->,cyan]
            (a)edge[out=55,looseness=2.5,in=-15](a)
            (b)edge[out=-55,in=-15](a)
            ;
            \path[thick,->,lightgray]
            (d)edge (e)
            ;
            \path[thick,->,dashed,lightgray]
            (e)edge (c)
            ;
            \path[thick,->,lightgray,dashed]
            (a)edge[out=35,in=150](e)
            (b)edge[out=-95,in=150](e)
            ;
         \end{scope}	
        \end{tikzpicture}
    \end{center}
\end{example*}

The reduct generalizes the SETAF-reduct (and the AF-reduct) in that for each BSAF $(A,R,S)$ with $S=\emptyset$ (and $|H|=1$ for all $(H,t)\in R$)
the two reduct notions coincide. 

We show that the reduct is compatible with union for conflict-free and closed sets. 

\begin{restatable}{proposition}{PROPreductUnion}
        	\label{prop_bsaf_reduct_union}
        	Let $\BF=(A,R,S)$ be a BSAF, $E\subseteq A$ closed and conflict-free in $\BF$,$E'\subseteq A^E$ closed and conflict-free in $\BF^E$. Then $\BF^{E\cup E'} = (\BF^E)^{E'}$.
\end{restatable} 

Let us inspect whether our newly defined reduct behaves as expected with respect to modularization $\pM$.
Towards proving this property, we first observe that the reduct guarantees to preserve the closedness of a set of arguments.
\begin{restatable}{proposition}{propReductPreservesClosed}
    \label{prop:reduct-preserves-closed}
    Let $ \BF=(A,R,S)$ be a BSAF, then for each $E\subseteq A,E'\subseteq A^E$ it holds that
    \begin{itemize}
        \item $E\in \cf(F)$ and closed $ \BF$ and $E'\in\cf(\BF^E)$ and closed in $ \BF^E\ \Rightarrow\ E\cup{E'}$ closed in $ \BF$
        \item $E, E\!\cup\! E'\in \cf(F)$ and closed in $\BF\ \Rightarrow\ E'$ closed in $ \BF^E$
    \end{itemize}
\end{restatable}
We show that all semantics in $\Sigma$ satisfy modularization 
(cf.~$\pM$ in Section~\ref{subsec:global des}), using our newly defined reduct.  

\begin{restatable}{proposition}{PropModularizationNormal}\label{prop:modularization normal semantics}
    $\sigma$ satisfies modularization $\pM$ for $\sigma\in\Sigma$.
\end{restatable}

\section{Strong Admissibility}
\label{sec:strong}

Following the investigation of the abstract case in \cite{DBLP:conf/comma/Caminada14,baumann2016stadmchara}, we lift strong admissibility to BSAFs, before we turn our attention to the consequences this has for ABA.

\subsection{Strong Admissibility for BSAFs}

Towards lifting strong admissibility to BSAF, lets recall that on AF it is defined recursively, where a set $E$ is strongly admissible, if each arguments $a\in E$ is defended by a strongly admissible set $E'\subseteq E\setminus\{a\}$~(cf.~Def~\ref{def:strong_adm_af}).
The support relation introduces a particularity that prevents us from using directly a similar recursion for BSAF: Clearly, a strongly admissible set should be closed,
yet if strong admissibility is defined recursively on itself, as it is in AF, we would need a closed set on each step of the defense.
Let us look at an example to illustrate this counter-intuitive behavior.

\begin{example}
    Consider the following two BSAFs $\BF$ and $\BF'$:%
    
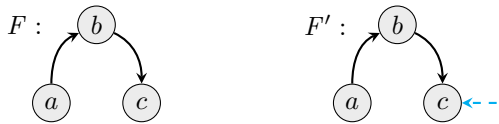
\vspace{-10pt}
    \begin{center}
		\begin{tikzpicture}[scale=0.8,>=stealth]
        \begin{scope}[xshift=0cm]
		\path
        (\distance-1.1,.85)node (D){}
        (0,0)node[arg] (b){$b$}
        (-1.2,0)node(F){$\BF:$}
		(.5*\distance,-\distance*0.866)node[arg] (c){$c$}
		(-.5*\distance,-\distance*0.866)node[arg] (a){$a$}
        ;
        
        \path[thick,->]
        (a)edge[out=90,in=-155](b)
        ;
        
        \path[thick,->]
        (b)edge[out=-25,in=90](c)
        ;
        \end{scope}
        
        \begin{scope}[xshift=5cm]
		\path
        (\distance-1.1,.85)node (D){}
        (0,0)node[arg] (b){$b$}
        (-1.2,0)node(F){$\BF':$}
		(.5*\distance,-\distance*0.866)node[arg] (c){$c$}
		(-.5*\distance,-\distance*0.866)node[arg] (a){$a$}
        ;
        
        \path[thick,->]
        (a)edge[out=90,in=-155](b)
        ;
        
        \path[thick,->]
        (b)edge[out=-25,in=90](c)
        ;
        \path[thick,->,cyan,dashed]
        (.5*\distance+0.9,-\distance*0.866)edge(c)
		;
        \end{scope}
        
		\end{tikzpicture}
    \end{center}
    In $\BF$ we expect $\emptyset$ and $\{a\}$ to be strongly admissible, further $\{a,c\}$ is strongly admissible, since $c$ is defended by the strongly admissible set $\{a\}$.
    In $\BF'$, however, the same set $\{a,c\}$ is not strongly admissible, if we require $c$ to be defended by a strongly admissible set, since $\{a\}$ is not closed.
    This is counterintuitive, since additional support on an extension should not speak against it.
    
\end{example}

We therefore decouple closedness and strong defense: an extension is strongly defended if it is defended by a smaller strongly defended extension, but does not need to be closed.

\begin{definition}
    \label{def_pstadm_recursive}
    Let $ \BF=(A,R,S)$ be a BSAF. A set of arguments $E\subseteq A$ is \emph{strongly defended} ($E\in\sd(\BF)$) iff $E$ is conflict-free and for every $a\in E$ there exists a strongly defended subset $E'\subseteq E\setminus\{a\}$, which defends $a$.
\end{definition}

We are now ready to define strong admissibility on BSAF.
 
\begin{definition}
    \label{def_stadm_recursive}
    Let $ \BF=(A,R,S)$ be a BSAF. A set of arguments $E\subseteq A$ is \emph{strongly admissible} ($E\in\sadm(\BF)$) iff $E\in\sd(\BF)$ and $E$ is closed.
\end{definition}

Like for AFs, \cite{baumann2016stadmchara}, we can now give a constructive characterization of strongly admissible extensions. 

\begin{restatable}{proposition}{PROPstadmDefinitions}\label{prop:stadmDefinitions}
    Let $ \BF=(A,R,S)$ be a BSAF, and $E\subseteq A$. Then $E\in\sad(\BF)$ iff $E$ is conflict-free, closed and there exists a finite sequence of pairwise disjoint sets $E_1,...,E_n$ such that $E_1=\emptyset$, $E=\bigcup\limits_{i=1}^n E_i$ and for each $i\geq 1$ it holds that $E_i$ is defended by $\bigcup\limits_{j<i} E_j$.
\end{restatable}

In the same vein we are able to proof the satisfaction of modularization $\pM$ wrt. strong admissible semantics.

\begin{restatable}{proposition}{PROPstadmModularization}
\label{prop_stadm_modularization}
Strongly admissible semantics satisfies modularization $\pM$. 
\end{restatable}

We will now discuss desirable properties specifically for for strong admissibility. 
We observe that strongly admissible semantics are a subset of admissible semantics, \ie, strengthening $\pSTR$ is satisfied. 
Likewise, the complete containment property $\pCC$ that formalizes that each strongly admissible set is contained in each complete extension holds for strongly admissible semantics for non-flat ABA.
\begin{restatable}{proposition}{PROPstadmPropertiesOne}
    \label{prop_stadm_properties}
    Let $ \BF=(A,R,S)$ be a BSAF. Then,
    \begin{enumerate}
    \item $\sadm(\BF)\subseteq \adm(\BF)$; and
    \item $E\in\sad(\BF)$ and $E'\in\com(\BF)$ implies $E\subseteq E'$.
    \end{enumerate}
\end{restatable}
As a consequence, each BSAF has at most 
one complete strongly admissible extension; 
and if such a complete extension exists, it is also the unique grounded and the unique $\subseteq$-maximal strongly admissible extension.

\begin{restatable}{proposition}{PROPstadmPropertiesTwo}
    \label{prop_stadm_properties_two}
    Let $ \BF=(A,R,S)$ be a BSAF. Then,
    \begin{enumerate}
         \item $|\stadm(\BF)\cap\com(\BF)|\leq 1$; and
         \item if $E\in\stadm(\BF)\cap\com(\BF)$, then $E$ is the unique grounded extension of $\BF$ ($E\in\grd(\BF)$) and the unique subset-maximal strongly admissible extension.
    \end{enumerate}
\end{restatable}
In contrast, unique relative maximum $\pURM$ and unique maximum $\pUM$ are not satisfied, 
as illustrated by the following example.
This example also shows that the complete strongly admissible extension does not always exist. It also demonstrates that the union of two strongly admissible sets may fail to be strongly admissible, thereby violating another property discussed in~\cite{DBLP:conf/comma/Caminada14}.

\begin{example}\label{ex:counter ex für unique max stradm}
    Consider the following BSAF $\BF$
    \begin{center}
		\begin{tikzpicture}[scale=0.8,>=stealth]
    		\path
            (\distance-1.1,.85)node (D){}
            (0,0)node[arg] (c){$c$}
    		(-\distance*0.866,.5*\distance)node[arg] (b){$b$}
    		(-\distance*0.866,-.5*\distance)node[arg] (a){$a$}
            (\distance,0)node[arg] (d){$d$}
    		(\distance*1.866,.5*\distance)node[arg] (e){$e$}
    		(\distance*1.866,-.5*\distance)node[arg] (f){$f$}
            ;
            
            \path[thick,->,cyan,dashed]
            (a)edge[out=55,in=180](c)
            (b)edge[out=-55,in=180](c)
            ;
            
            \path[thick,->]
            (d)edge(c)
            (e)edge[in=120,out=-120](f)
            (f)edge[in=-60,out=60](e)
            (e)edge(d)
            (f)edge(d)
            ;
		\end{tikzpicture}
    \end{center}
    In $\BF$, the sets $\emptyset,\{a\},\{b\},\{a,b,c,e\},\{a,b,c,f\}$ are admissible; 
    the latter two are complete, grounded, and preferred. 
    The strongly admissible sets are $\sad(\BF)=\{\emptyset,\{a\},\{b\}\}$.
    \begin{itemize}[align=left]
        \item $\pURM$ Both $\{a,b,c,e\}$ and $\{a,b,c,f\}$ have two \mbox{$\subseteq$-maximal} strongly admissible subsets: $\{a\}$ and $\{b\}$;
        \item $\pUM$ Both $\{a\}$ and $\{b\}$ are $\subseteq$-maximal in $\sad(\BF)$;
        \item Their union $\{a,b\}$ is not strongly admissible.
    \end{itemize}
\end{example}

As the example shows, strongly admissible extensions take over an important function of the grounded semantics, \ie\ they identify arguments with a stronger justification. In the abstract setting, the unique grounded extension contains only arguments, whose defense can be traced back to the empty set, and which therefore do not rely on cycles for their defense. In non-flat ABA, this is not true for members of grounded extensions in general, but still holds for arguments accepted under strongly admissible semantics. In general there can be more than one subset-maximal strongly admissible extension, which warrants the definition of \emph{strongly preferred} and \emph{strongly complete} semantics as a BSAF-semantics in its own right. 

\begin{definition}\label{def:strong BSAF semantics}
	Let $\BF$ be a BSAF and let $E\in\sadm(\BF)$. 
	\begin{itemize}
		\item 
		$E\in\scomp(\BF)$ iff $E$ contains every assumption it defends; 
		\item 
		$E\in\spref(\BF)$ iff $E$ is $\subseteq$-maximal in $\sadm(\BF)$;
	\end{itemize}
\end{definition}
We omit strongly grounded semantics since it coincides with strongly complete semantics (as observed below Proposition~\ref{prop_stadm_properties}, $\scomp(\BF)$ has at most one member).
We write $\Sigma^s=\{\sadm,\scomp,\spref\}$ to denote the family of strongly admissible semantics. 
Note that the framework $\BF$ above has two complete 
extensions $\{a,b,c,e\}$, and $\{a,b,c,f\}$. The strong variant does not have an extension in $\BF$, \ie\, $\scom(\BF)=\emptyset$.
Example~\ref{ex:counter ex für unique max stradm} shows that the fundamental lemma $\pF$ is not satisfied, and that in general $\sprf(\BF)\subseteq\scom(\BF)$ does not hold.
\smallskip

\begin{example*}{ex:counter ex für unique max stradm}
\label{ex:counter-ex to FL for sadm}
 The strongly admissible set $\{a\}$ defends the unattacked argument $b$. Since $\{a,b\}$ is not closed it is not strongly admissible. Thus the fundamental lemma $\pF$ is violated and we have no grounded extension, instead $\{a\},\{b\}$ are our two strongly preferred extensions.
\end{example*}

\subsection{Concequences for (flat and non-flat)  ABA}

In the last section we proved the (un)satisfiability of several desiderata that pose requirements of the ABA semantics directly in BSAF. 
We utilize these findings to define and investigate semantics based on strong admissibility for ABAFs.

\begin{definition}
    Given an ABAF $\D=\tuple{\lit,\rules,\asms,\contraryempty}$ and a semantics $\sigma\in\Sigma^s$, then $E\in\sigma(\D)$ iff $E\in\sigma(\BF_\D)$.
\end{definition}
Due to the close correspondence of ABAFs and BSAFs, 
the results of the previous section directly transfer to ABA.
Strongly admissible semantics for ABA satisfy modularization $\pM$ since the corresponding BSAF semantics does, as shown in Proposition~\ref{prop_stadm_modularization}; moreover, 
strengthening  $\pSTR$ and complete containment $\pCC$ is satisfied by Proposition~\ref{prop_stadm_properties}. 
For the remaining cases, the counter-examples carry over. 

\begin{restatable}{theorem}{propPropertiesStrongAdmABA}
    \label{prop:properties_strong_adm_aba}
    The strongly admissible semantics for ABA satisfies $\pM$, $\pSTR$, and $\pCC$,  but does not satisfy $\pF$, $\pSR$, $\pNE$, $\pMC$, $\pUM$, nor $\pURM$.
\end{restatable}
While central properties are satisfied, 
our results show that, as anticipated, the family of strongly admissible semantics admits in some aspects undesired behavior.

For the flat ABA fragment, these issues do not occur. 
Flat ABA, directly corresponds to SETAF, meaning that the instantiation $\BF_\D=\tuple{A,R,S}$ of a flat ABAF $\D$ (via Def.~\ref{def:non-flat collective inst}) has no support relations, \ie\, $S=\emptyset$.
It holds that $\pM$, $\pSTR$ and $\pCC$ are satisfied since any SETAF is a BSAF. Further, if we take a look at the counter-examples used to show non-satisfaction, we notice all of them use at least one support. It turns out that all $\pF$, $\pSR$, $\pNE$, $\pMC$, $\pUM$, and $\pURM$ are indeed satisfied for flat ABA. 

\begin{restatable}{theorem}{propPropertiesStrongAdmSETAF}
    \label{prop:properties_strong_adm_aba_setaf}
    The strongly admissible semantics for flat ABA and SETAFs satisfy all of $\pM$, $\pSTR$, $\pCC$, $\pF$, $\pSR$, $\pNE$, $\pMC$, $\pUM$, and $\pURM$.
\end{restatable}

\section{Weak Admissibility}
\label{sec:weak}

In contrast to strong admissibility our goal is now to accept as much as we reasonably can.
In this section, we generalize the definition of weak admissibility by~\citeauthor{BlumelKU24}~\shortcite{BlumelKU24} to non-flat ABA.
\subsection{Weak Admissibility for BSAFs}

The $E$-reduct for BSAFs gives us the tools to generalize weak admissibility.
Note that the definition is recursive, but well-defined as in each recursion step the reduct contains fewer arguments and we deal only with finite BSAFs.
\begin{definition}\label{def_setaf_wadm_two}
		Let $ \BF=(A,R,S)$ be a BSAF, $E\subseteq A$ a set of arguments, and $ \BF^E=(A^E,R^E,S^E)$ its $E$-reduct. 
		Then $E$ is called \emph{weakly admissible} in $ \BF$ ($E\in \wadm( \BF)$) iff
	\begin{enumerate}
		\item $E\in \cf( \BF)$, $E$ closed and
		\item for each $(T,h)\in R$ with $h\in E$, and
		$T\cap E^+_R=\emptyset$ 
		it holds $\nexists E' \in \wadm( \BF^E)$ s.t.\ $T\cap A^E\subseteq E'$.
	\end{enumerate}
\end{definition}
As it is the case for SETAFs, the tail of a joint attack has to be part of a weakly admissible set (of the reduct) as a whole for it to be considered an attack one has to defend against. In BSAFs, we additionally have to take closedness into consideration. On the one hand, we limit the set of weakly admissible extensions to closed sets. 
On the other hand, this allows us to ignore even more involved types of unreasonable attacks, \eg\, indirectly conflicting sets of attackers.  
\begin{example}
\label{ex:self-attacking-sets} Compare the SETAF $F_0$ and BSAF $F$:
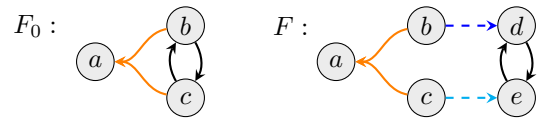
\vspace{-10pt}
\begin{center}
    	\begin{tikzpicture}[scale=0.8,>=stealth]
    		\path
    		(\distance*-1.5,0) node {$F:$}
            (\distance-1.1,.85)node (D){}
            (0,0)node[arg] (b){$b$}
    		(0,-\distance*0.8)node[arg] (c){$c$}
    		(\distance,-\distance*0.8)node[arg] (e){$e$}
    		(-\distance,-\distance*0.4)node[arg] (a){$a$}
      		(\distance,0)node[arg] (d){$d$}
            ;
            
            \path[thick,->]
            (e) edge[bend left] (d)
            (d) edge[bend left] (e) ;
            \path[thick,->,dashed,blue]
            (b) edge (d)
            ;
            \path[thick,->,dashed,cyan]
            (c) edge (e)
            ;
            \path[thick,->,orange]
            (b)edge[out=190,in=0](a)
            (c)edge[out=170,in=0](a)
            ;
            
         \begin{scope}[xshift=-4cm]
             		\path
    		(-2.5,0) node {$F_0:$}
            (\distance-1.1,.85)node (D){}
            (0,0)node[arg] (b){$b$}
            (0,-\distance*0.8)node[arg] (c){$c$}
    		(-\distance,-\distance*0.4)node[arg] (a){$a$}
            ;
            
            \path[thick,->]
                        (b) edge[bend left] (c)
                        (c) edge[bend left] (b) ;
            
                        \path[thick,->,orange]
                        (b)edge[out=190,in=0](a)
                        (c)edge[out=170,in=0](a)
                        ;
         \end{scope}	
        \end{tikzpicture}
    \end{center} 
    In both frameworks $\{a\}$ is attacked by $\{b,c\}$, and weakly admissible. 
    In $F_0$ the joint attack does not fire, because $\{b,c\}$ is not conflict-free. 
    
    At first sight, nothing seems to be wrong with $\{b,c\}$ in $F$, the set is not only conflict-free, it is even unattacked. However, $b$ and $c$ support conflicting arguments, so even though the closed subsets $\{b,d\}$ and $\{c,e\}$ are weakly admissible in the reduct $F^{\{a\}}$, there is no closed and conflict-free set containing both $b$ and $c$, so the joint attack on $a$ can be ignored.   
\end{example}
Closedness neutralizes attacks from indirectly self-conflicting sets under weak admissibility.
Next, we show that closedness of attackers already holds by definition.

\begin{restatable}{proposition}{propRestrictionToClosedAttackers}
    Let $ \BF=(A,R,S)$ be a BSAF, let $E\subseteq A$ be a set of arguments, and $ \BF^E=(A^E,R^E,S^E)$ its $E$-reduct. 
	Then $E\in\wadm( \BF)$  iff $E$ is conflict-free, closed, and for every closed set $E'$ which attacks $E$ it holds that either $E$ attacks $E'$ or $(E'\setminus E)\notin\wadm( \BF^E)$.
\end{restatable}

The notion of weak defense, as well as, weakly complete, weakly grounded and weakly preferred semantics generalize to BSAFs in the natural way.

\begin{definition}
	\label{def:weak defense}
	Let $ \BF = (A,R,S)$ be a BSAF and let $E,X\subseteq A$. 
	We say $E$ \emph{weakly defends} $X$ (abbr.\ $E$ w-defends $X$) if for each $(T,h)\in R$  with $h\in X$, and for every closed $E'$ with $T\subseteq E'$ one of the following two conditions hold: 
	\begin{itemize}
		\item $E$ attacks $E'$, or
		\item the following two conditions hold simultaneously: 
		\begin{enumerate}
			\item there is no $E^*\in\adm^w( \BF^E)$ with $E'\subseteq E \cup E^*$, 
			\item there is some $X'$ s.t.\ $X\subseteq X'\in\adm^w( \BF)$. 
		\end{enumerate}
	\end{itemize}
\end{definition}

\begin{definition}
	\label{def:wcom semantics}
	Let $\BF$ be a BSAF and let $E\in\wadm(\BF)$:
	\begin{itemize}
		\item $E$ is \emph{weakly complete}, $E\in\wcom(\BF)$, iff for each $X\supseteq E$ s.t.\ $E$ w-defends $X$, we have $E = X$,  
		\item $E$ is \emph{weakly preferred}, $E\in\wpref(\BF)$, iff $E$ is maximal wrt.\ $\subseteq$ in $\wadm(\BF)$, 
		\item $E$ is \emph{weakly grounded}, $E\in\wgrd(\BF)$, iff $E$ is minimal wrt.\ $\subseteq$ in $\wcom(\BF)$. 
	\end{itemize}
\end{definition}
We denote by $\Sigma^w=\{\wadm,\wcomp,\wgrd,\wpref\}$ the family of weakly admissible semantics. 

Turning now to modularization~$\pM$, we find that an even stronger result can be proved. 
As for AFs~\cite{BBU2020Modularization}, also the other direction holds.

\begin{restatable}{proposition}{PropModularizationWeakAdm}
    \label{prop:modularization wadm}
    Let $\sigma\in\Sigma^w$, and $\BF=\tuple{A,R,S}$, then for 
    each $E\subseteq A,E'\subseteq A^E$ we have
        \begin{itemize}
        \item $E\in\sigma( \BF)$, $E'\in\sigma( \BF^E)\ \Rightarrow\ E\cup{E'}\in\sigma( \BF)$
        \item $E, E\cup E'\in\sigma( \BF)\ \Rightarrow\ E'\in\sigma( \BF^E)$
    \end{itemize}
\end{restatable}

\subsection{Consequences for ABA}
We are ready to introduce weak admissibility for ABA. As in the case of strong admissibility, we define the semantics with respect to the instantiated BSAF. 

\begin{definition}
\label{def:WADMforABA}
	Let $\D$ be an ABAF and $ \BF_\D$ be the instantiated BSAF. 
	For any $\sigma^w\in \Sigma^w$ 
	we let 
	$\sigma^w(\D) = \sigma^w(\BF_\D)$. 
\end{definition}

Note that the semantics faithfully generalize the weakly admissible semantics of the flat ABA fragment~\cite{BlumelKU24}.

We examine our novel semantics with respect to the desiderata stated in Section~\ref{sec:desiderata}. Observe first, that weak admissibility satisfies liberalization $\pL$, \ie\, classic admissibility is a special case of weak admissibility. Modularization $\pM$ is also preserved, which indicates our generalization is well-behaved for the most part. The indifference towards the presence of self-attacking assumptions $\pP$, however, a key feature of weak admissibility in AFs and SETAFs, is no longer given in BSAFs. Due to the additional requirement of closedness and the sensitivity towards indirect conflicts accompanying it, self-attackers influence the set of weakly admissible extensions indirectly via their incoming supports.

\begin{example}
Consider $\BF=(\{a\},\{(\{a\},a)\},\{(\emptyset,a)\})$. This BSAF contains a single self-attacking assumption $a$ which is supported by the empty set, rendering the empty set not acceptable under weak admissibility since it is not closed. Now if $a$ is removed, $\emptyset$ immediately becomes a weakly admissible set.
\end{example}
 It turns out that weak admissibility for BSAFs, and consequently for ABA, takes a more differentiated look at self-attackers than for SETAFs/flat ABA. On the one hand, self-attackers compromise the assumptions supporting them, on the other hand, they do not impact the acceptance of assumptions attacked by them under weak admissibility. The later is validated by the fact that weak admissibility satisfies paradoxical attacks/supports $\pPR$ in BSAFs. 

Observe that, as for classic admissible semantics, the fundamental lemma $\pF$ is not satisfied for weakly admissible semantics and the family of weak semantics admits unwanted behavior. 
We summarize our findings below. 
\begin{restatable}{theorem}{ThmSatisfiabilityWadm}
\label{thm:ABA weak adm results}
    The weakly admissible semantics for ABA satisfies $\pL$, $\pPR$, and $\pM$,
     and does not satisfy $\pF$, $\pP$, $\pNE$, $\pSR$,$\pMC$.
\end{restatable}

\newcommand{\scl}{\cl_\Gamma}
\section{Fixing Strong Admissibility (and why Fixing Weak Admissibility Fails)}\label{sec:FL-sensitive semantics}

The undesired behavior of ABA semantics has received quite some attention in the literature recently~\cite{HeyninckA24,BertholdRU24}. 
To overcome some of the issues, \citeauthor{BertholdRU24}~(\citeyear{BertholdRU24}) propose alternatives to the classical semantics that reinstate some of the desired properties of the semantics.
Inspired by their investigations, 
we discuss ways how to address the observed shortcomings in our setting. 
To do so, we focus on the so-called $\Gamma$-semantics as they address issues of admissible-based semantics by modifying the closure.
In this section, we focus on BSAF semantics; as in the previous sections, the results transfer to ABA as well.

We recall $\Gamma$-closure below.
Intuitively, an argument $a$ only counts as supported by a set $E$ if $E$ is strong enough to defend $a$ against each attack ($\Gamma$ is defined in Definition~\ref{def:BSAF gamma and more}).%
\begin{definition}
	Given a BSAF $\BF=\tuple{A,R,S}$, $E\subseteq A$, and $a\in A$. 
	Then $E$ $\Gamma$-supports $a$ iff $a\in\cl(E)$ and $a\in \Gamma(E)$;  
	$E$ is $\Gamma$-closed iff $E$ contains all arguments it $\Gamma$-supports.
	By $\cl_\Gamma(E)$ we denote the $\Gamma$-closure of $E$.
\end{definition}
The $\Gamma$-closure induces versions of  admissible, preferred, complete, and grounded semantics, which we denote by $\sigma_\Gamma$.
We state the definition of $\Gamma$-admissible semantics; the remaining semantics are defined analogously.
\begin{definition}
Let $\BF$ be a BSAF. $E\!\in\!\cf(\BF)$ is \emph{$\Gamma$-admissible} ($E\in\gadm(\BF)$) iff $E$ defends itself and is $\Gamma$-closed.
\end{definition}
As discussed in~\cite{BertholdRU24}, $\Gamma$-admissible semantics satisfy many of the desired semantics properties, especially for admissible-based semantics. 
The objective of this section is to introduce $\Gamma$-closure as a fix to the strong and weak admissible semantics to address some of the undesired properties for our semantics. As it turns out, however, $\Gamma$-semantics lack an important prerequisite, especially for weak semantics: the modularization property is not satisfied.  
We demonstrate this issue below.%
\begin{example}\label{exm:counter ex mod sgadm}
Consider the following BSAF $\BF=\tuple{A,R,S}$ and its reduct wrt.\ $E=\{d\}$ below.
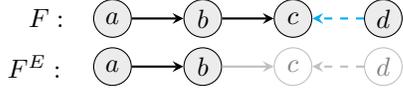
\begin{center}
		\begin{tikzpicture}[scale=0.8,>=stealth]
		\node at (-1.1,0) {$F:$};
		\path
		(0,0) node[arg] (a) {$a$}
        (1.5,0) node[arg] (b){$b$}
		(3,0) node[arg] (c){$c$}
  		(4.5,0)node[arg] (d){$d$}
        ;
        
        \path[thick,->]
        (a)edge(b)
        ;
        \path[thick,->,cyan,dashed]
        (d)edge(c)
        ;
        \path[thick,->]
        (b)edge(c)
        ;
        
        \begin{scope}[yshift=-0.8cm]
        	\node at (-1.3,0) {$F^E:$};
    		\path
    		(0,0) node[arg] (a) {$a$}
            (1.5,0) node[arg] (b){$b$}
    		(3,0) node[argd] (c){\color{gray}$c$}
      		(4.5,0)node[argd] (d){\color{gray}$d$}
            ;
            \path[thick,->]
            (a)edge(b)
            ;
            \path[thick,->,lightgray,dashed]
            (d)edge(c)
            ;
            \path[thick,->,lightgray]
            (b)edge(c)
            ;
        \end{scope}	
\end{tikzpicture}
\end{center}
$E$ is $\Gamma$-closed since $E$ does not defend $c$ although it is in the closure of $E$. 
Thus, $E$ is $\Gamma$-admissible.  
Let $E'=\{a\}$.
Thus, $E$ is weakly $\Gamma$-admissible in $\BF$, $E'$ is weakly $\Gamma$-admissible in $\BF^E$.
However, $E\cup E'$ is not weakly $\Gamma$-admissible in $\BF$.
\end{example}

Without modularization the $\Gamma$-closure is not well suited for fixing weakly admissible semantics, because the idea of rejecting an attacker that is not acceptable in the reduct has to be justified wrt. the framework as a whole. 
For strong admissibility, on the other hand, the adaptation of $\Gamma$-closedenss offers several desirable results, as we discuss below. 
First, let us define the $\Gamma$-version of strong admissibility, and the semantics based on strong admissibility.
\begin{definition}
    Let $\BF$ be a BSAF; a set $E\in\cf(\BF)$ is \emph{strongly $\Gamma$-admissible} ($E\in\sgadm(\BF)$) iff $E\in\sd(\BF)$ and $E$ is $\Gamma$-closed.
    Further, given $E\in\sgadm(\BF)$: 
    	\begin{itemize}
    		\item 
    		$E\in\sgcom(\BF)$ iff $E$ contains every assumption it defends; 
    		\item 
    		$E\in\sggrd(\BF)$ iff $E$ is $\subseteq$-minimal in $\sgcom(\BF)$;
    		\item 
    		$E\in\sgpref(\BF)$ iff $E$ is $\subseteq$-maximal in $\sgadm(\BF)$;
    	\end{itemize}
\end{definition}
Note that strong $\Gamma$-admissible semantics are guaranteed to return some extension since $\cl_\Gamma(\emptyset)$ can be closed.  
\begin{restatable}{proposition}{PROPsgadmNonEmpty}
    \label{prop:sgadm non-empty}
    $\sgadm(\BF)\neq \emptyset$ for each BSAF $\BF$.    
\end{restatable}
As a result, non-emptiness $\pNE$ is satisfied by strongly $\Gamma$-admissible semantics. 

Next, we give a constructive definition of strongly $\Gamma$-admissible semantics.
\begin{restatable}{proposition}{PROPstgadmDefinitions}
    \label{def_stgadm_constructive}
    Let $ \BF=(A,R,S)$ be a BSAF. A set of arguments $E\subseteq A$ is strongly $\Gamma$-admissible iff it is conflict-free, $\Gamma$-closed and there exists a finite sequence of pairwise disjoint sets $E_1,...,E_n$ such that $E_1=\emptyset$, $E=\bigcup\limits_{i=1}^n E_i$ and for each $i\geq 1$ it holds that $E_i$ is defended by $\bigcup\limits_{j<i} E_j$.
\end{restatable}
We also obtain the following weaker version of the fundamental lemma $\pF$ for strongly $\Gamma$-admissible semantics.
\begin{itemize}[align=left]
    \item[$\pWF$] Weakened Fundamental Lemma: If $S\!\in\!\adm(\BF)$ defends $a$, then there exists $S'\in\adm(\BF)$, s.t.\ $S\cup\{a\}\!\subseteq\! S'$.
\end{itemize}
Given a strongly $\Gamma$-admissible set $S$ and an argument $a$ it defends, we can use the constructive definition of the semantics to compute the $\Gamma$-closure of $S\cup \{a\}$ in an iterative way. 
In each step, we add all arguments that lie in the $\Gamma$-closure and are not already contained in $S$. The resulting sequence satisfies the requirements from Proposition~\ref{def_stgadm_constructive}.
\begin{restatable}{proposition}{PropWeakFundamentalLemmaStrongGamma}
    \label{prop:weak FL for sadm and sgadm}
    If $E\in\sgadm(\BF)$ defends $a$, then there exists $E'\in\sgadm(\BF)$, such that $E\cup\{a\}\subseteq E'$.
\end{restatable}
Note that the weakened fundamental lemma $\pWF$ is not satisfied for $\sadm$, as Example~\ref{ex:counter-ex to FL for sadm} shows: $a$ defends the unattacked argument $b$ but there is no strongly admissible extension that contains both.

As a consequence of Proposition~\ref{prop:weak FL for sadm and sgadm}, we obtain that each BSAF contains a unique maximal $\Gamma$-admissible extension. 
Also, strongly $\Gamma$-admissible semantics  satisfies strengthening $\pSTR$ and complete containment $\pCC$. 
Overall, strongly $\Gamma$-admissible semantics satisfy several properties that are not satisfied wrt.\ strong admissibility, as summarized below.

\begin{restatable}{theorem}{propPropertiesStrongGammaAdm}
    \label{prop:properties_strong_g_adm}
    The strongly $\Gamma$-admissible semantics for BSAF satisfies $\pSTR$, $\pNE$, $\pWF$, $\pURM$, $\pUM$, $\pCC$, $\pSR$, and $\pMC$,
    but does not satisfy $\pM$ and $\pF$.
\end{restatable}

\section{Conclusion}
This work introduces strong and weak admissibility for non-flat ABA.
We generalize the notion of reduct to BSAF and show that modularity is satisfied by standard, weak, and strong admissibility. 
We furthermore propose semantics based on strong $\Gamma$-admissibility, which satisfy several desirable properties.
Overall, our results reveal a fundamental trade-off between general semantics properties 
and a sufficiently well-behaved notion of closedness, \eg, strongly $\Gamma$-admissible semantics are $\Gamma$-closed but not modular.

Our findings contribute to ongoing work in ABA and abstract argumentation in a multitude of ways, touching open issues regarding existence and computation of extensions, principle satisfaction and interrepresentability of formalisms wrt.\ three well-established notions of admissibility.
Our investigations
pave the way for sequential computation of extensions~\cite{CaminadaH24SADMalgortihms,BengelThimm22} for a large variety of semantics in BSAF and general ABA. 
Moreover, successfully capturing weak admissibility for general ABA provides valuable insights towards a native reduct notion for ABA and an ABA-semantics satisfying long-standing rationality postulates like non-interference~\cite{BorgS18NonIntereference} in the future. 
The novel families of semantics can be beneficial for several ABA applications since they can model real world scenarios where the classical semantics may be too strict or not strict enough,\eg, in the planning approach by \citeauthor{DBLP:conf/prima/Fan18a}~\shortcite{DBLP:conf/prima/Fan18a} which uses flat ABA or in the causal discovery setting which uses non-flat ABA~\cite{DBLP:conf/kr/0002RT24}.

Our results demonstrate that no existing semantics satisfies all weak admissibility desiderata in the general case. As shown in Theorem~\ref{thm:ABA weak adm results}, paradoxical assumptions cannot be avoided, and attempts to resolve this using $\Gamma$-admissible semantics fall short due to the lack of modularity. Identifying a suitable middle ground that ensures more of the desired properties remains an interesting and challenging direction for future work. One promising candidate is the
$\Delta$-semantics proposed in \cite{BertholdRU24} to address related issues of complete-based semantics under strong and weak admissibility.

\section*{Acknowledgments}

The research reported here was partially supported by the Deutsche Forschungsgemeinschaft (grant 550735820). Furthermore, the authors acknowledge the financial support by the Federal Ministry of Education and Research of Germany and by Sächsische Staatsministerium für Wissenschaft, Kultur und Tourismus in the programme Center of Excellence for AI-research ``Center for Scalable Data Analytics and Artificial Intelligence Dresden/Leipzig´´, project identification number: \hyperref{https://scads.ai/}{}{}{ScaDS.AI}.

\bibliographystyle{kr}

\clearpage

\appendix

\section{General Desiderata and Their Satisfiability for Classical Dung Semantics}
In this section, we discuss details about the global desiderata.
First, we give the proof of the following result.

\PropMCiffSR*
\begin{proof}
Let $E\in \pref(\D)$. By $\pSR$, $E\in \com(\D)$. 
Let $E'\in \com(\D)$ s.t.\ $E'\subseteq E$ and $E'$ is maximal in $\com(\D)$. 
By definition, $E'\in \adm(\D)$. Thus $E'=E$ is maximal in $\com(\D)$. 

Let $E\in \com(\D)$ be $\subseteq$-maximal in $\com(\D)$.
By definition, $E\in \adm(\D)$.
let  $E'\in \pref(\D)$ with $E'\subseteq E$. 
By $\pSR$, $E'\in\com(\D)$. Thus $E=E'$.

\end{proof}

Next, we discuss the satisfaction of the general desiderata in the case of non-flat ABA.  
As already pointed out by \citeauthor{BertholdRU24}~\shortcite{BertholdRU24} most of the fundamental principles listed in Section~\ref{subsec:global des} are violated for the standard semantics.  

We give counterexamples below. 

\begin{example}
    \label{ex:SR_MC_unsat}
    Consider the following BSAF $\BF=\tuple{A,R,S}$, where $A=\{a,b\}$, $R=\{\tuple{\{b\},b}\}$ and $S=\{\tuple{\{a\},b}\}$. The set $E$ is admissible in $\BF$ and defends $a$, yet $E\cup\{a\}=\{a\}\notin\adm(\BF)$. Hence $\pF$ is not satisfied.
    Further, $\adm(\BF)=\{\emptyset\}$, and therefore $\pref(\BF)=\{\emptyset\}$, yet $\com(\BF)=\emptyset$. Hence $\pref(\BF)\not\subseteq\com(\BF)$ -- neither $\pSR$ nor $\pMC$ hold.
\end{example}

\begin{example}
    Consider the following BSAF $\BF=\tuple{A,R,S}$, where $A=\{a\}$, $R=\{\tuple{\emptyset,a}\}$ and $S=\{\tuple{a},a\}$. Then $\sigma(\BF)=\emptyset$ for all $\sigma\in\{\adm,\com,\pref,\grd\}$, contradicting $\pNE$
\end{example}
We summarize satisfiability and unsatisfiability of the classical semantics. 
The proof for modularization is given in the preceding Section (cf.~Proposition~\ref{prop:modularization normal semantics}). The examples given prove the unsatisfiability of the remaining desiderata.
 
\begin{restatable}{proposition}{propPropertiesNormalBSAF}
    \label{prop:properties_normal_BSAF}
    The admissible semantics for BSAF satisfies $\pM$, but does not satisfy $\pF$, $\pSR$, $\pNE$, nor $\pMC$.
\end{restatable}

\section{Omitted Proofs Section~\ref{sec:reduct}}
\PROPreductUnion*
        	
        	\begin{proof} 
        	Let $\BF^E=(A^E,R^E,S^E)$. 
        	\begin{description}
        	\item[$A^{E\cup E'}=(A^E)^{E'}$] By \ref{prop:reduct-preserves-closed} we have $E\cup E'$ is closed and conflict-free in $\BF$, so $A^{E\cup E}=A\setminus{(E\cup E'\cup (E\cup E')^+_R)}$. It is left to show that $(E\cup E')^+_R=E^+_R\cup E'^+_{R^E}$. ($\subseteq$) Let $a\in(E\cup E')^+_R$, then there exists some attack $(T,a)\in R$ such that $T\subseteq E\cup E'$. If $T\subseteq E$, then $a\in E^+_R$, otherwise by definition of the reduct $(T\setminus E,a)\in R^E$, since $E$ is closed, but $T\setminus E\subseteq E'$, so $a\in E'^+_{R^E}$.
        	($\supseteq$) $a\in E^+_R$ is clear, let $a\in E'^+_{R^E}$. Then there exists an attack $(T,a)\in R^E$ such that $T\subseteq E'$. By definition of the reduct there thus either exists an attack $(T',a)\in R$ such that $T'\setminus T\subseteq E$ or a support $(T,b)\in S$ with $b\in E^+_R$. But then $a\in E'$, which contradicts that $E'$ is conflict-free, so it can only be the case that the attack $(T',a)$ exists, thus $a\in (E\cup E')^+_R$. 
        	
        	\item[$R^{E\cup E'}=(R^E)^{E'}$] Follows directly from $E'$ being conflict-free in $\BF^E$ and from $E\cup E'$ being closed by Prop. \ref{prop:reduct-preserves-closed}
        	
        	\item[$S^{E\cup E'}=(S^E)^{E'}$] Follows directly from the definition of the reduct. Since $E\cup E'$ is closed by Prop. \ref{prop:reduct-preserves-closed}, we have $\cl(E\cup E')\cap E^+_R=\emptyset$.
        	\end{description}
        	\end{proof}

\propReductPreservesClosed*
\begin{proof}
        (1)	Suppose to the contrary there was some $h\in A, (T,h)\in S$ such that $h\notin E\cup E'$ and $T\subseteq E\cup E'$. Then either $h\in\BF^E\setminus E'$ or $h\in \cl(E)^+=E^+$.
        	
        	(Case 1: $h\in\BF^E$) Then $(T\setminus E,h)\in S^E$ is a support in $\BF^E$. Since $T\subseteq E\cup E'$ we have $T\setminus E\subseteq E'$, so since $E'$ is closed in $\BF^E$ we have $h\in E'$. Contradiction.
        	
        	(Case 2: $h\in E^+$) We further distinguish the cases
        	\begin{itemize}
        	\item $T\subseteq E$, then $E=\cl(E)$ is not conflict-free. Contradiction.
        	\item $T\setminus E\neq\emptyset$, then by Definition \ref{def_bsaf_reduct} there exists an attack $(T\setminus E,t)\in R^E,t\in T\setminus E$ in $\BF^E$. Since $T\setminus E\subseteq E'$, $E'$ is not conflict-free in $\BF^E$. Contradiction.
        	\end{itemize}  
        	
       (2) 	Suppose to the contrary there was some $h\in A^E, (T,h)\in S^E$ such that $h\notin E'$ and $T\subseteq E'$. Then, since $E$ is closed, by definition of $S^E$ there exists some $T'\subseteq E$ such that $(T'\cup T,h)\in S$. But then $E\cup E'$ is not closed. Contradiction. 	
        	\end{proof}
        	
\PropModularizationNormal*
\begin{proof}\hphantom{asd}
    \begin{itemize}
        \item[$\adm$] Consider BSAF $\BF$ and assume $E\in\adm( \BF)$, $E'\in\adm( \BF^E)$. There cannot be attacks between $E$ and $E'$ in $\BF$: If $E$ attacks $E'$ in $\BF$, then the attacked argument of $E'$ would not be in the reduct. Vice versa, if $E'$ attacks $E$ in $\BF$, $E$ would have to defend itself against $E'$ to be admissible, which brings us back to the first case. Further, $E\cup E'$ defends itself in $\BF$: $E$ defends itself in $\BF$ per assumption, and $E'$ defends itself against any argument that is not attacked by $E$, since it is admissible in the reduct. Any argument attacking $E'$ in $\BF^E$ that is jointly attacked by $E\cup E'$ in $\BF$, by construction of the reduct is attacked by $E$ in $\BF^E$.
        Finally, $E\cup E'$ is closed: We have $\cl_\BF(E\cup E')=\cl_\BF(E)\cup\cl_{\BF^E}(E')$, since any supports that require arguments from both $E$ and $E'$ in $\BF$ only require those of $E'$ in $\BF^E$. Since $E$ and $E'$ are both closed in their respective framework, so is $E\cup E'$.
        In conclusion $E\cup E'\in\adm(\BF)$.
        \item[$\comp$] Consider BSAF $\BF$ and assume $E\in\comp(\BF)$, $E'\in\comp( \BF^E)$, then $E\cup E'$ is conflict-free, closed and defends itself, by the same reasoning of the $\adm$ case. We only need to prove that $E\cup E'$ contains all arguments it defends. This can be argued similarly to the closure. Let us assume $E\cup E'$ defends an argument $a$ in $\BF$. If $a$ is defended using only arguments from either $E$ or $E'$, then $a\in E$ or $a\in E'$ directly. If on the other hand $a$ is defended by arguments from both $E$ and $E'$, then $E'$ defends $a$ in $\BF^E$, hence $a\in E'$. $E\cup E'\in\com(\BF)$ holds.
        \item[$\grd$] Consider BSAF $\BF$, $E\in\grd(\BF)$ and $E'\in\grd(\BF^E)$. We assume toward contradiction that $E'\neq\emptyset$. Then $E'$ contains an argument $a$ that is defended by $\emptyset$ in $\BF^E$, hence $a$ is defended by $E$ in $\BF$, therefore $E$ is not the grounded extension of $\BF$. Contradiction! The set $E'$ must be empty, therefore $E\cup E'=E\in\grd(\BF)$ holds.
        \item[$\pref$] The property can be proven analogously to the $\grd$ case. The reduct of a preferred extension cannot have a non-empty preferred extension.
    \end{itemize}
\end{proof}

\section{Omitted Proofs Section~\ref{sec:strong}}

To aid the next proof, we define an \emph{alternative reduct} of BSAF that deviates from the regular reduct in two ways:
\begin{itemize}
    \item It does not take into account the closure of extensions, \ie\, any `$\cl(E)$' in the original definition instead is `$E$'.
    \item It does not create attacks between arguments that jointly support an argument in $E^+$.
\end{itemize}

\begin{definition}\label{def_bsaf_reduct_alternative}
    Given a BSAF $\BF=(A,R,S)$ and $E\subseteq A$, the \emph{alternative $E$-reduct of $\BF$} is the BSAF $\BF_{alt}^E=(A_{alt}^E,R_{alt}^E,S_{alt}^E)$, with
    \begin{align*}
        A_{alt}^E = &\ A \setminus E^\oplus_R\\
    	R_{alt}^E = &\ \{(T\setminus E, h) \mid  T\cap E^+_R=\emptyset, \\
    	& \phantom{\ \{(T\setminus E,t)\mid\quad} (T,h)\in R^E, h\in A^E
    	\}\\
        S_{alt}^E =
    	&\ \{(T\setminus E, h) \mid  T\cap E^+_R=\emptyset, \\
    	& \phantom{\ \{(T\setminus E,t) \mid\quad} (T,h)\in S, h\in A^E
    	\}
    \end{align*}
\end{definition}

\begin{example}
    Consider the following BSAF $\BF$ and set of arguments $E$. The regular reduct $\BF^E$ removes $c$ and $e$, as they are in the closure of $E$ and attacked by $E$ respectively, and adds joined self-attacks between $a$ and $b$. In contrast, $F_{alt}^E$ retains the argument $c$, and does not have joined self-attacks between $a$ and $b$.
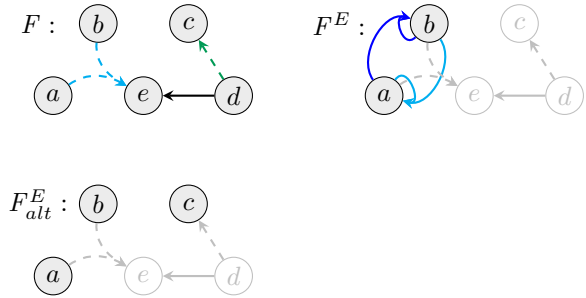
    \begin{center}
    	\begin{tikzpicture}[scale=0.8,>=stealth]
    		\path
    		(-1,0) node {$\BF:$}
            (\distance-1.1,.85)node (D){}
            (0,0)node[arg] (b){$b$}
    		(\distance,0)node[arg] (c){$c$}
    		(.5*\distance,-\distance*0.8)node[arg] (e){$e$}
    		(-.5*\distance,-\distance*0.8)node[arg] (a){$a$}
      		(1.5*\distance,-\distance*0.8)node[arg] (d){$d$}
            ;

            \path[thick,->]
            (d)edge (e)
            ;
            \path[thick,->,dashed,ForestGreen]
            (d)edge (c)
            ;
            \path[thick,->,cyan,dashed]
            (a)edge[out=35,in=150](e)
            (b)edge[out=-95,in=150](e)
            ;
         \begin{scope}[xshift=5.5cm]
             		\path
    		(-1.5,0) node {$\BF^E:$}
            (\distance-1.1,.85)node (D){}
            (0,0)node[arg] (b){$b$}
    		(\distance,0)node[argd] (c){$c$}
    		(.5*\distance,-\distance*0.8)node[argd] (e){$e$}
    		(-.5*\distance,-\distance*0.8)node[arg] (a){$a$}
      		(1.5*\distance,-\distance*0.8)node[argd] (d){$d$}
            ;
            
            \path[thick,->,blue]
            (a)edge[out=125,in=165](b)
            (b)edge[out=-135,looseness=2.5,in=165](b)
            ;
            \path[thick,->,cyan]
            (a)edge[out=55,looseness=2.5,in=-15](a)
            (b)edge[out=-55,in=-15](a)
            ;
            \path[thick,->,lightgray]
            (d)edge (e)
            ;
            \path[thick,->,dashed,lightgray]
            (d)edge (c)
            ;
            \path[thick,->,lightgray,dashed]
            (a)edge[out=35,in=150](e)
            (b)edge[out=-95,in=150](e)
            ;
         \end{scope}	
         \begin{scope}[yshift=-3cm]
             \path
    		(-1,0) node {$\BF_{alt}^E:$}
            (\distance-1.1,.85)node (D){}
            (0,0)node[arg] (b){$b$}
    		(\distance,0)node[arg] (c){$c$}
    		(.5*\distance,-\distance*0.8)node[argd] (e){$e$}
    		(-.5*\distance,-\distance*0.8)node[arg] (a){$a$}
      		(1.5*\distance,-\distance*0.8)node[argd] (d){$d$}
            ;

            \path[thick,->,lightgray]
            (d)edge (e)
            ;
            \path[thick,->,dashed,lightgray]
            (d)edge (c)
            ;
            \path[thick,->,lightgray,dashed]
            (a)edge[out=35,in=150](e)
            (b)edge[out=-95,in=150](e)
            ;
         \end{scope}
        \end{tikzpicture}
    \end{center}
\end{example}

The alternative reduct has the following positive properties wrt. strong defense. In particular it satisfies the inverted direction of $\pM$.

\begin{lemma}
    \label{le:alt reduct preserves strong defense}
    Let $ \BF=(A,R,S)$ be a BSAF, then for each $E\subseteq A$, $E'\subseteq A_{alt}^E$, we have
    \begin{enumerate}
        \item $E,E\cup E'\in\sd(\BF)\ \Rightarrow\ E'\in\sd(\BF_{alt}^E)$
        \item $E\in\sd(\BF)$ then there exists a finite sequence of pairwise disjoint sets $E_1,...,E_n$ such that $E_1=\emptyset$, $E=\bigcup\limits_{i=1}^n E_i$ and for each $i\geq 1$ it holds that $E_i$ is defended by $\bigcup\limits_{j<i} E_j$.
    \end{enumerate}
\end{lemma}
\begin{proof}\hphantom{asd}
    \begin{enumerate}
        \item 
        Let $\BF=\tuple{A,R,S}$, $E,E\cup E'\in\sd(\BF)$. 
        We show that for any $a\in E'$ there exists subset $E^*\subseteq E'\setminus\{a\}$ defending $a$ in $\BF_{alt}^E$, \suth $E^*\in\sd(\BF_{alt}^E)$. Proof by induction over the size $|E'|=n$ of $E'$.
        
        (Base case) $E'=\emptyset$, $n=0$, trivial.
        
        (Induction step) Let $|E'|=n$, $a\in E'$. Since $E\cup E'\in\sd(\BF)$, there exists a set $E''\subseteq E\cup E'\setminus\{a\}$ which defends $a$ in $\BF$ and satisfies $E''\in\sd(\BF)$. If $E''\subseteq E$, then $a$ is unattacked and therefore defended by the empty set in $\BF_{alt}^E$. Otherwise, we set $E^*:=E''\setminus E$ we may use the induction hypothesis: $E\in\sd(\BF)$ and $E\cup E^*\in\sd(\BF)$, hence $E^*\in\sd(\BF_{alt}^E)$ and defends $a$ in $\BF_{alt}^E$.

        \item
        Proof by induction over the size of $E$.
        
        (Base Case) $E=\emptyset$, then the sequence $E_1=\emptyset$ of length $1$ satisfies the equivalent definition of being strongly defended is trivially satisfied.
        
        (Induction Step) Let $E\in\sd(\BF)$ and $|E|=n$. Let $a\in E$ some argument in $E$. Then there exists a strongly defended subset $E^*\subseteq E\setminus\{a\}$ which defends $a$. It holds that $a\notin E^*$, so $|E^*|<n$, therefore by the induction hypothesis there exists a sequence of pairwise disjoint sets $E_1,...,E_{m-1}$ such that $E_1=\emptyset$, $E^*=\bigcup\limits_{i=1}^{m-1} E_i$ and for each $i\geq 1$ it holds that $E_i$ is defended by $\bigcup\limits_{j<i} E_j$. Now let $E_m=E\setminus{E^*}$.
        By (1) the set $E_m$ is strongly defended in $\BF_{alt}^{E^*}$.
        Since $|E_m|<n$, we can apply the induction hypothesis. There exists a sequence of pairwise disjoint sets $E'_1$, \dots, $E'_p$ such that $E'_1=\emptyset$, $E_m=\bigcup\limits_{i=1}^{p} E'_i$ and for each $i\geq 1$ it holds that $E'_i$ is defended by $\bigcup\limits_{j<i} E'_j$ in $\BF_{alt}^{E^*}$.
        Since $E'_1=\emptyset$ defends $E'_2$ in $\BF_{alt}^{E^*}$, $E^*$ defends $E'_2$ in $\BF$.
        We arrive at the end.
        The sequence $E_1$,~\dots,~$E_{m-1}$,~$E'_1$,~\dots,~$E'_p$ satisfies the alternative definition of being strongly defended.
    \end{enumerate}
\end{proof}

\PROPstadmDefinitions*
\begin{proof}
    \begin{description}
        \item[$(\Rightarrow)$] $E\in\sad(\BF)$ iff $E$ closed and $E\in\sd(\BF)$. If we apply Lem~\ref{le:alt reduct preserves strong defense} (2), we directly see that $E$ satisfies the alternative definition of strong admissibility.
        
        \item[$(\Leftarrow)$] Proof by induction over the length $n$ of the sequence $E_1,...,E_n$. 
        
        (Base Case) $n=1$. If $E_1=\emptyset$ is closed, then it is also strongly admissible, for the empty set contains no arguments that need to be defended by a subset or could be in conflict.
        
        (Induction Step) Suppose we have a sequence $E_1,...,E_n$ of pairwise disjoint sets, such that $E_1=\emptyset$, and for each $i\geq 1$ it holds that $E_i$ is defended by $\bigcup\limits_{j<i} E_j$. Then by the induction hypothesis for each $i<n$ we have $\bigcup\limits_{j\leq i} E_j=E^*$ is strongly defended. In particular, $\bigcup\limits_{j\leq (n-1)} E_j$ is a strongly defended subset of $E=\bigcup\limits_{i=1}^n E_i$. By the last condition of the definition we have $E=\bigcup\limits_{i=1}^n E_i$ is closed and conflict-free. It is left to show that for each $a\in E$ there exists a strongly defended subset of $E$ defending $a$. Let $a\in E_i$ for some $i$. Then $a$ is defended by $\bigcup\limits_{j\leq i-1} E_j=E^*$, which is strongly defended. Since the $E_i$ are pairwise disjoint, $E^*\subseteq E\setminus\{a\}$.
    \end{description}
\end{proof}

\PROPstadmModularization*
\begin{proof}
        Let $\BF=\tuple{A,R,S}$, $E\in\sadm(\BF)$, $E'\in\sadm(\BF^E)$. 
        Then by Prop.~\ref{prop:reduct-preserves-closed} holds that $E\cup E'$ is closed, furthermore $E\cup E'$ is conflict-free. 
        To show that for each $a\in E\cup E'$ there exists a subset $E^*\subseteq(E\cup E')\setminus\{a\}$ that defends $a$ and is strongly defended in $\BF$, we prove the following statement by induction:
        
        (Induction hypothesis) Let $E\in\sadm(\BF), E'\in\sd(\BF^E)$. Then $E\cup E'\in\sd(\BF)$.
        
        Let $a\in E\cup E'$. If $a\in E$ we can just choose the respective subset of $E\setminus\{a\}$. For the other case we use induction over the size $n$ of $E'$.
        
        (Base case) $E'=\emptyset$, $n=0$, then $a$ in $E$.
        
        (Induction step) Let $|E'|=n$, $a\in E'$. 
        Since $E'\in\sd(\BF^E)$, there exists a set $E^*\subseteq E'\setminus\{a\}$ which defends $a$ in $\BF^E$ and satisfies $E^*\in\sd(\BF^E)$.
        By the induction hypothesis $E\cup E^*\in\sd(\BF)$, hence $E\cup E^*\in\sd(\BF)$ and by Def. \ref{def_bsaf_reduct} $E\cup E^*$ defends $a$ in $\BF$. Any attacker of $a$ in $\BF$ is either attacked by $E$, since $E$ is closed, or has a subset in $\BF^E$ and $a$ is defended by $E^*$.
        
\end{proof}

\PROPstadmPropertiesOne*
\begin{proof}
    \hphantom{asd}
    \begin{enumerate}
        \item Let $E\in\stadm(\BF)$. By Def.~\ref{def_stadm_recursive} $E$ is conflict-free, closed and each member of $E$ is defended by a subset of $E$. Since $\Gamma$ is monotonic, $E$ defends all of its members, so $E\in\adm(\BF)$. 
        \item Assume two sets $E,E'$, \suth $E\in\sad(\BF)$ and $E'\in\com(\BF)$.
        By definition $E\in\sad(\BF)$ implies that there are disjoint sets $E_1,\dots,E_n$, \suth $E_1=\emptyset$, $E=\bigcup\limits_{i=1}^n E_i$ and for each $i\geq 1$ it holds that $E_i$ is defended by $\bigcup\limits_{j<i} E_j$. Since $E_1=\emptyset$, $E'$ defends $E_1$. Since additionally $E'\in\com(\BF)$, we obtain $E_1\subseteq E'$. We repeat this argument for all $E_i$ with $1\leq i\leq n$, and obtain $E\subseteq E'$.
    \end{enumerate}
\end{proof}

\PROPstadmPropertiesTwo*
\begin{proof}
    \hphantom{asd}
    \begin{enumerate}
    \item Assuming that two set $E,E'$, \suth $E,E'\in\sad(\BF)$ and $E,E'\in\com(\BF)$, we may apply (2) of Proposition~\ref{prop_stadm_properties} twice to arrive at $E=E'$.
    \item Assume $E\in\sad(\BF)\cap\com(\BF)$ and $E'\in\grd(\BF)$. Clearly $E'\subseteq E$, since $E'$ by definition is subset minimal in $\com(\BF)$.
    On the other hand (2) of Proposition~\ref{prop_stadm_properties} implies that $E\subseteq E'$. Hence $E=E'$.
    \end{enumerate}
\end{proof}

Towards proving Theorems~\ref{prop:properties_strong_adm_aba} and \ref{prop:properties_strong_adm_aba_setaf}, we will first provide the following result for BSAFs.
\begin{restatable}{theorem}{propPropertiesStrongAdm}
    \label{prop:properties_strong_adm}
    The strongly admissible semantics for BSAF satisfies $\pM$, $\pSTR$, and $\pCC$ but does not satisfy $\pF$, $\pSR$, $\pNE$, $\pMC$, $\pUM$, nor $\pURM$.
\end{restatable}
\begin{proof}\hphantom{asd}
    \begin{itemize}
        \item $\pM$ holds for $\sad$ due to Prop.~\ref{prop_stadm_modularization}.
            The proof for $\sigma\in\{\comp,\grd,\pref\}$ is analogous to the proof of modularization for standard BSAF semantics, cf.\ Proposition~\ref{prop:modularization normal semantics}.
        \item $\pSTR$ holds due to Prop.~\ref{prop_stadm_properties}.
        \item $\pCC$ hold due to Prop.~\ref{prop_stadm_properties}.
        \item The unsatisfiability of $\pNE$ is due to the following BSAF: $\BF=\tuple{A,R,S}$, where $A=\{a\}$, $R=S=\{\tuple{\emptyset,a}\}$.
        \item The unsatisfiability of $\pF$, $\pSR$, $\pURM$, $\pUM$ and $\pMC$ is due to the counter-example Exm.~\ref{ex:counter ex für unique max stradm}.
    \end{itemize}
\end{proof}

\propPropertiesStrongAdmABA*
\begin{proof}
    The results regarding $\pM$ $\pSTR$, $\pCC$, $\pF$, $\pSR$,
    $\pNE$, $\pMC$, $\pUM$, and $\pURM$ are follow directly from Theorem~\ref{prop:properties_strong_adm}.
\end{proof}

\propPropertiesStrongAdmSETAF*
\begin{proof}
    \begin{itemize}
        \item The satisfaction of $\pM$ $\pSTR$, and $\pCC$ follows directly from Thm.~\ref{prop:properties_strong_adm}.
        \item $\pF$: Given SETAF $\SF=\tuple{A,R}$, $E\in\sad(\SF)$ and $a\in A$, \suth $E$ defends $a$.
        Surely, $E\cup\{a\}$ is conflict-free. Using Prop.~\ref{prop:stadmDefinitions}, we get: 
        there exists a finite sequence of pairwise disjoint sets $E_1,...,E_n$ such that $E_1=\emptyset$, $E=\bigcup\limits_{i=1}^n E_i$ and for each $i\geq 1$ it holds that $E_i$ is defended by $\bigcup\limits_{j<i} E_j$.

        Choose $E_{n+1}=\{a\}$, then we have the same statement for $E\cup\{a\}$, hence, $E\cup\{a\}\in\sad(\SF)$.
        \item $\pNE$: $\emptyset\in\sad(\SF)$ is guaranteed, hence $\sad(\SF)\neq\emptyset$, and hence also $\spref(\SF)\neq\emptyset$. Further $\Gamma$ is monotonous and $\SF$ finite, hence $\scom(\SF)\neq\emptyset$ and therefore $\sgrd(\SF)\neq\emptyset$.
        \item $\pSR$: Assume toward contradiction $E\in\spref(\SF)$, \suth $E$ defends $a$ and $E\notin\scom(\SF)$. Then $E\cup\{a\}$ admissible, hence $E$ is not $\subseteq$-maximal in $\sad(\SF)$, contradiction!
        \item $\pUM$ and $\pURM$ can be proven in the same manner. In SETAF the following implication holds: $E_1,E_2\in\sad(\SF)\Rightarrow E_1\cup E_2\in\sad(\SF)$. Hence assuming two different $\subseteq$-maximal strongly admissible sets, directly leads to contradiction.
    \end{itemize}
\end{proof}

\section{Omitted Proofs of Section~\ref{sec:weak}}

\propRestrictionToClosedAttackers*
\begin{proof}
    Clearly $E$ is conflict-free and closed for both directions, it is left to show that $E\in\wadm( \BF)\Rightarrow\forall E'\ \text{closed},E'\text{attacks }E: E'\setminus E\notin\wadm( \BF^E)\vee E\text{attacks }E'$ and the reverse is true. We show this by contraposition.
    \begin{description}
    \item[($\Rightarrow$)] Suppose there exists some closed $E'$ attacking $E$ with $E'\setminus{E}\in\wadm( \BF^E)$ and $E$ does not attack $E'$. Then, since $E$ is closed, $E'\cap\cl(E)^+=\emptyset$ and there exists some attack $(T,h)\in R$ with $T\subseteq E',h\in E$ and $T\cap\cl(E)^+=\emptyset$. Furthermore, since $T\subseteq E$ it follows $T\cap A'\subseteq E'\cap A'=E'\setminus E$. Since $E'\setminus E\in\wadm( \BF)$, $E$ is attacked by a weakly admissible set from the reduct and therefore $E$ is not weakly admissible.
    \item[($\Leftarrow$)] Suppose $E\notin\wadm( \BF)$, but $E$ is conflict-free and closed. Then there exists an attack $(T,h)\in R$ such that $h\in E$, $T\cap E^+=\emptyset$ and $T\setminus E\subseteq E^*$ for some $E*\in\wadm( \BF^E)$. Since $E^*\in\wadm( \BF^E)$, it is conflict-free and closed in $ \BF^E$, therefore by Definition \ref{def_bsaf_reduct}, there is no $(T',g)\in S$ such that $g\in E^+$ and $T'\subseteq E\cup E^*$. So $E'=E\cup E*$ is closed in $ \BF$ and thus a closed set$E'$ attacking $E$ with $E'\setminus E=E^*\in\wadm( \BF)$ exists.
    \end{description}
\end{proof}

\PropModularizationWeakAdm*

\begin{proof}
        \begin{description}
        \item[$\sigma=\wadm$] 
        	(1) We show this by induction over the number of arguments in $\BF$. The base case (empty BSAF) is trivial. So suppose the claim holds for each BSAF with $n = |A|$ and we are given $\BF = (A,R,S)$ where $|A|=n+1$. 
        	Suppose $E\in\adm^w(\BF)$ and $E'\in\adm^w(\BF^E)$. 
        	We show weak admissibility of $E\cup E'$ in $\BF$. 
        	        	
        	For the special case $E=\cl(E)=E^+= \emptyset$, we have $\BF^E=\BF$, so $E\cup E'=E'\in\wadm(\BF)$.  
        	
        	In any other case $\BF^E$ contains strictly less arguments and the induction hypothesis applies to $\BF^E$. 
        	
        	(closed) Since $E,E'$ are closed and conflict-free in $\BF$, resp. $\BF^E$, by Proposition \ref{prop:reduct-preserves-closed} $E\cup E'$ is closed.  
        	
        	(conflict-free) 
        	Striving for a contradiction suppose there is some 
        	$(T,h)\in R$ with $h\in E\cup E'$ s.t.\ 
        	$T\subseteq E\cup E'$. We use that $E=\cl(E)$ when talking about $\BF^E$.
        	
        	(case 1: $h\in E$) 
        	In this case, $E\cup E'$ attacks $E$.  
        	If $T\subseteq E$ this contradicts conflict-freeness of $E$. 
        	So assume $T\nsubseteq E$. 
        	Next we argue that $T\cap E^+ = \emptyset$: 
        	i) if $E$ attacks $t\in T$ with $t\in E$, then $E$ is again not conflict-free; contradiction; 
        	ii) if $E$ attacks $t\in T$ with $t\in E'$, then $E'$ does not occur in the reduct as a full set; contradiction. 
        	Consequently, we have $T\cap E^+ = \emptyset$. 
        	Moreover, $T\cap A'\subseteq E'$ holds in the reduct where $E'\in\adm^w(\BF^E)$ by definition. This implies $E\notin \adm^w(\BF)$; a contradiction. 
        	
        	(case 2: $h\in E'$) 
        	If $T\subseteq E$, then some argument in $E'$ does not occur in $\BF^E$; contradiction. 
        	If not, then we construct an attack $(T\setminus E,h)$ in the reduct, because as we saw above, $T\cap E^+ = \emptyset$ holds and $h\in E'\subseteq  A'$. 
        	But then $E'\notin \cf(\BF^E)$; contradiction. 
        	
        	We deduce $E\cup E'\in\cf(\BF)$. 
        	
        	(weak admissibility)
        	Suppose to the contrary there is some $E''\in\adm^w(\BF^{E\cup E'})$ attacking $E\cup E'$, \ie\, there exists $(T,h)\in R$ such that $T\subseteq E''$ and $h\in E\cup E'$. We use that $E=\cl(E)$ and $E\cup E'=\cl(E\cup E')$ when talking about $\BF^E,\BF^{E\cup E'}$.
        	
        	(case 1: $h\in E$)
        	We have $\BF^{E\cup E'}=(\BF^E)^{E'}$ by Lemma \ref{prop_bsaf_reduct_union}.
        	
        	Now, by the induction hypothesis, modularization holds in $\BF^E$, so since $E'\in \wadm(\BF^E)$ and $E''\in\wadm((\BF^E)^{E'})$, we have $E'\cup E''\in\wadm(\BF^E)$. So $E$ is attacked by a weakly admissible extension from its reduct, but $E$ is weakly admissible. Contradiction. 
        	
        	(case 2: $h\in E'$) 
        	Again 
        	$\BF^{E\cup E'} = (\BF^E)^{E'}$ by Lemma \ref{prop_bsaf_reduct_union}, 
        	and consequently, by modularization in the smaller BSAF $\BF^E$, 
        	$E''\in\adm^w((\BF^E)^{E'})$; this, however, is a contradiction to $E'\in\adm^w(\BF^E)$. 
        	
        	(2)We show that
        	if $E\in\wadm(\BF)$ and $E\cup E'\in\wadm(\BF)$, then $E'\in\wadm(\BF^E)$. 
        	
        	(closed) Since $E,E\cup E'$ are closed and conflict-free in $\BF$, by Proposition \ref{prop:reduct-preserves-closed} $E'$ is closed in $\BF^E$.  
        	        	
        	(conflict-free)
        	Suppose $(T',h)\in R^E$ occurs in the reduct $\BF^E$ with $h\in E'$ and $T'\subseteq E'$. 
        	Then either there exists $(T,h)\in R$ of the form $T' = T\setminus E$ and consequently, $T\subseteq E\cup E'$ in $\BF$ which implies $E\cup E'\notin\cf(\BF)$; a contradiction. 
        	
        	Or there exists a support $(T'\cup\{h\},a)$ with $a\in E^+$, then $E\cup E'$ cannot be both closed and conflict-free. Contradiction.    
        	
        	(weak admissibility)
        	Suppose $E''\in\adm^w( (\BF^E)^{E'} )$ attacks $E'$. Since $E'$ is both closed and conflict-free, by Lemma \ref{prop_bsaf_reduct_union}, we have $\BF^{E\cup E'}=(\BF^E)^{E'}$.
        	Therefore $E''\in\adm^w(\BF^{E\cup E'})$ attacks $E\cup E'$ which contradicts 
        	$E\cup E'\in\adm^w(\BF)$. 
        	
        \item[$\sigma=\wpref$] (1) Let $E\in\wprf(\BF)$ and $E'\in\wprf(\BF^E)$. Then by modularization of $\wadm$ we have $E\cup E'\in\wadm(\BF)$ is a weakly admissible superset of $E$. But $E$ is weakly preferred and therefore subsetmaximal, so $E\cup E'=E\in\wprf(\BF)$
        
        (2) If $E,E\cup E'\in\wpref(\BF)$, then $E'=\emptyset$ and by modularization of $\wadm$ it holds $E'\in\wadm(\BF^E)$. Suppose to the contrary there was some proper weakly admissible superset $E''\supset E'$ in $\BF^E$, then by modularization $E\cup E''\in\BF^E$, so $E\cup E'\notin\wpref(\BF^E)$. Contradiction.
        
        \item[$\sigma=\wcomp$] (1) Let $E\in\wcom(\BF)$ and $E'\in\wcom(\BF^E)$. Then by modularization of $\wadm$ we have $E\cup E'\in\wadm(\BF)$ is weakly admissible. Suppose now, there was some superset $X\supseteq E\cup E'$ which is weakly defended by $E\cup E'$, and $X\setminus (E\cup E')\neq\emptyset$. 
         Then for every closed set $Y$ attacking $X$ either $(E\cup E')\cap Y\neq \emptyset$ (Case 1), or $Y\setminus (E\cup E')$ has no superset $Y^*\in\wadm(\BF^{E\cup E'})$, and there exists some $X'$ such that $X\subseteq X'\in\wadm(\BF)$(Case 2).
         
         (Case 1) By definition of the reduct if $E\cup E'$ attacks some $Y$, since $E$ and $E'$ are closed and conflict-free, and $E^+_R\cap E'=\emptyset$, then either $Y\subseteq E^+_R$ or $E'$ attacks $Y\cap A^E$ in $\BF^E$. But then $E'$ weakly defends $X\setminus E$ in $\BF^E$, so $E'$ is not weakly complete, contradiction. 
         
         (Case 2) Then $X\setminus E$ is weakly defended by $E'$ in $\BF^E$, because by modularization $X'\setminus E\in\wadm(\BF^E)$ and any attacking set $Y\setminus E$ of $X\setminus E$ is not the subset of any $Y^*\in\wadm((\BF^E)^{ E'})$, since Prop.~\ref{prop_bsaf_reduct_union} applies. But then again, $E'$ is not weakly complete in $\BF^E$. Contradiction.  
                
                (2) If $E,E\cup E'\in\wcom(\BF)$, then by modularization of $\wadm$ it holds $E'\in\wadm(\BF^E)$. Suppose to the contrary there was some proper superset $X\supsetneq E'$, which is weakly defended by $E'$ in $\BF^E$. Then for every closed set $Y$ attacking $X$ either $E'^+_{R^E}\cap Y\neq \emptyset$ (Case 1), or $Y\setminus E'$ has no superset $Y^*\in\wadm((\BF^E)^{E'})$, and there exists some $X'\in\wadm(\BF^E)$ such that $X\subseteq X'$(Case 2).
                
                (Case 1) Since $E'$ is weakly admissible, it is conflict-free, so by definition of the reduct if $E'^+_{R^E}\cap Y\neq \emptyset$, then $(E\cup E')^+_R\cap Y\neq \emptyset$. But then $E\cup E'$ weakly defends $X$. Contradiction.
                
                (Case 2) By modularization $X'\cup E\in\wadm(\BF)$ with $E\cup E'\subseteq X\cup E\subseteq X'\cup E$ and any attacking set $Y$ of $X \cup E$ is not the subset of any $Y^*\in\wadm(\BF^{E\cup E'})$, since Prop.~\ref{prop_bsaf_reduct_union} applies. But then $E\cup E'$ weakly defends $X\cup E$, so $E\cup E'$ is not weakly complete. Contradiction.  
        
        \item[$\sigma=\wgrd$] (1) Let $E\in\wgrd(\BF)$ and $E'\in\wgrd(\BF^E)$. Then by modularization of $\wcom$ we have $E\cup E'\in\wcom(\BF)$ is a weakly complete superset of $E$. If $E'$ is the empty set we have $E\cup E'=E$ and are done. Suppose to the contrary that $E'\neq\emptyset$. Since $E\in\wadm(\BF)$, it holds that $\emptyset\in\wadm(\BF^E)$ (because anything supported by the emptyset in the reduct is supported by $E$ in $\BF$ and $E$ is closed). If $E'\in\wgrd(\BF^E)$ is not empty, then the empty set is not weakly complete in $\BF^E$. So there exists some $X$ which is weakly defended by the empty set in $\BF^E$. But then $X\cup E$ is weakly defended by $E$ (see proof of modularization for weakly complete semantics above), so $E$ is not weakly complete. Contradiction.
                
                (2) If $E,E\cup E'\in\wgrd(\BF)$, then $E'=\emptyset$ and by modularization of $\wcom$ it holds $E'\in\wcom(\BF^E)$. Since $E'$ is empty, it is minimal. 	
        \end{description}
        \end{proof}

The following proposition show that the weakly admissible semantics are faithfully generalized.
\begin{restatable}{proposition}{PROPwsemanticsgeneralizecorrect}
    Let $SF=(A,R)$ be a SETAF, and let $\BF=(A,R,\emptyset)$ the corresponding BSAF with empty supports. Then $\sigma(SF)=\sigma(\BF)$ for $\sigma\in\Sigma^w$.
\end{restatable}
\begin{proof}
    Since there are no supports present, every set $E\subseteq a$ in $\BF$ is closed, \ie\, $E=\cl(E)$, and $SF^E$ contains the same arguments and attacks as $\BF^E$. Given that, the equality follows directly from the definitions of the resp. semantics.
\end{proof}

\ThmSatisfiabilityWadm*

\begin{proof}\hphantom{asd}
    
    \begin{itemize}[align=left]
        \item[$\pPR$] 
        \begin{description}
        \item[($\sigma=\wadm$)] 
        Let $(T,h)$ be a paradoxical attack in some BSAF $\BF=(A,R,S)$ and let $\BF'=(A,R\setminus\{(T,h)\},S)$. We will prove $\wadm(\BF)=\wadm(\BF')$ by induction over the size of $A$, with $|A|=0$, the empty BSAF being the trivial base case. For the induction step let $|A|=n$.
        \begin{description}
        \item[$(\subseteq)$] Suppose $E\in\wadm(\BF)$, then, since we remove an attack, $E$ is still closed and conflict-free in $\BF'$. It is left to show that $E$ weakly defends itself, \ie\, there is no attack $(T^*,h^*)$ with $h^*\in E$ and $T^*\setminus E\subseteq E^*$ for some $E^*\in\wadm(\BF'^E)$. If $E=\emptyset$ no such $h^*$ exists, so the statement holds. For $E\neq\emptyset$ the reduct $\BF^E$ contains strictly less arguments than $\BF$, so it satisfies the induction hypothesis. We can now distinguish the following cases (since $E$ is conflict-free it cannot be the case that $T\subseteq E$):
        
        \begin{description}
           \item[($T\cap E^+\neq\emptyset$ or $h\in E^\oplus$)] Then the attack is deleted in the reduct, so $BF^E=\BF'^E$ and thus $E^*\in\wadm(\BF'^E)$ iff $E^*\in\wadm(\BF^E)$ for any such $(T^*,h^*)$.
        
           \item[($T\subseteq A^E$ and $h\in A^E$)] Then by definition of paradoxical attacks there exist attacks $(T',t)$ for each $t\in T$ with $T'\subseteq T$, so $(T',t)\in R^E$. So $(T,h)$ is a paradoxical attack in $\BF^E$ and furthermore $\BF'^E=(A^E,R^E\setminus\{(T,h)\},S^E)$. By the induction hypothesis it therefore holds that $E^*\in\wadm(\BF'^E)$ iff $E^*\in\wadm(\BF^E)$ for any such $(T^*,h^*)$.
        
           \item[($T\cap E\neq\emptyset$ and $h\in A^E$)] Then $(T\setminus E,h)$ is a paradoxical attack in $\BF^E$, since for every $(T',t)\in R$ attacking a $t\in T$, as defined for paradoxical attacks, we have $(T'\setminus E,t)\in R^E$. In particular, since $E$ is conflict-free, for any such $T'$ we have $T'\cap E\neq \emptyset$. Furthermore, since deleting $(T,h)$ does not change the set $E^+$ we have $BF'^E=(A^E,R^E\setminus\{(T\setminus E,h)\},S^E)$, so by the induction hypothesis $E^*\in\wadm(\BF'^E)$ iff $E^*\in\wadm(\BF^E)$ for any such $(T^*,h^*)$.
        \end{description}
        
        \item[($\supseteq$)] Suppose $E\in\wadm(\BF')$. Then $E$ is closed in $\BF$, because no supports were changed, and $E$ is conflict-free, because if we suppose $T\subseteq E$, then there exists an attack $(T',t)\in R$ with $T'\subseteq E, t\in E$ and $h\notin T'$, so $E$ is not conflict-free. Contradiction. It is again left to show that $E$ weakly defends itself, \ie\, there is no attack $(T^*,h^*)$ with $h^*\in E$ and $T^*\setminus E\subseteq E^*$ for some $E^*\in\wadm(\BF^E)$. If $E=\emptyset$ no such $h^*$ exists, so the statement holds. For $E\neq\emptyset$ the reduct $\BF^E$ contains strictly less arguments than $\BF$, so it satisfies the induction hypothesis. We distinguish the same cases as above for all of which it holds that $E^*\in\wadm(\BF'^E)$ iff $E^*\in\wadm(\BF^E)$ for any such $(T^*,h^*)$. So $E\in\wadm(\BF)$.
        
        \end{description}
        
        Now let $(T,h)\in S$ be a paradoxical support, $\BF'=(A,R,S\setminus\{(T,h)\})$. We will prove $\wadm(\BF)=\wadm(\BF')$ by induction over the size of $A$, with $|A|=0$, the empty BSAF being the trivial base case. For the induction step let $|A|=n$.
        
        \begin{description}
        \item[$(\subseteq)$] Let $E\in\wadm(\BF)$. Then $E$ is closed and conflict-free in $\BF'$ because deleting a support does not add arguments to the closure of a set.It is left to show that $E$ weakly defends itself, \ie\, there is no attack $(T^*,h^*)$ with $h^*\in E$ and $T^*\setminus E\subseteq E^*$ for some $E^*\in\wadm(\BF'^E)$. If $E=\emptyset$ no such $h^*$ exists, so the statement holds. For $E\neq\emptyset$ the reduct $\BF^E$ contains strictly less arguments than $\BF$, so it satisfies the induction hypothesis. We can now distinguish the following cases (since $E$ is conflict-free it cannot be the case that $T\subseteq E$):
        
        \begin{description}
        \item[($T\cap E^+_R\neq\emptyset$ or $h\in E$)] Then $BF^E=BF'^E$ and consequently $E^*\in\wadm(\BF'^E)$ iff $E^*\in\wadm(\BF^E)$ for any such $(T^*,h^*)$.
        
        \item[($h\in E^+_R$)] Then for every $t\in T\setminus E$, which is not empty, we add an attack $(T\setminus E,t)$ in $\BF^E$ according to the definition of the reduct. These attacks do not necessarily occur in $\BF'^E$, let $R^*=R^E\setminus R'^E$ be the set of all added attacks, which are not present in $\BF^E$. We have $\BF'^E=(A^E,R'^E\cup R^*)$. We have to show $\wadm(\BF^E)=\wadm(\BF'^E)$. We proof this by induction over the size of $A^E$, with $|A^E|=0$ being the trivial base case. For the induction step let $|A^E|=n$. Now let $E'\in\wadm(\BF^E)$, then $E'$ is closed and conflict-free in $\BF'^E$. The same is true the other way around, if $E'\in\wadm(\BF'^E)$ then $E'$ is closed and conflict-free in $\BF^E$, because, since $(T,h)$ is a paradoxical support, by definition of the reduct there exist attacks $(T'\setminus E,t)\in R'^E$ for every $t\in T\setminus E$, so $T\setminus E\not\subseteq E'$. It is left to show, that $E'$ weakly defends itself in $\BF^E$ iff $E'$ weakly defends itself in $\BF'^E$.
        
         For that it suffices to show that if there exists an attack $(T'',a)\in R^E$ with $a\in E'$ and $T''\setminus E'\subseteq E''$ for some $E''\in\wadm((\BF^E)^{E'})$ then $(T'',a)\in R'^E$ and $E''\in\wadm((\BF'^E)^{E'})$ and the other way around.  If $E'=\emptyset$, then no such attack $(T'',a)$ can exist. Otherwise $(\BF^E)^{E'}$ and $(\BF'^E)^{E'}$ contain strictly less arguments than $\BF^E$ and $\BF'^E$, resp., so they satisfy the induction hypothesis and therefore $E''\in\wadm((\BF^E)^{E'})$ iff $E''\in\wadm((\BF'^E)^{E'})$. Furthermore, since $T''\subseteq E''$ and $E''$ is conflict-free, it cannot be the case that $T''=T'\setminus E$ for any of the additional attacks $(T'\setminus E,t)\in R^*$. So $(T'',a)\in R^E\cap R'^E$. 
         
        \item[($h\in A^E$ and $T\setminus E\subseteq A^E$)] Then $(T\setminus E,h)$ is a paradoxical support in $\BF^E$, because, for every $t\in T\setminus E$ there exists $(T'\setminus E,t)\in R^E$ with $T'\setminus E\neq\emptyset$, since $t\notin E^+$. So by the induction hypothesis $E^*\in\wadm(\BF'^E)$ iff $E^*\in\wadm(\BF^E)$ for any such $(T^*,h^*)$.
        \end{description}
        
        \item[$(\supseteq)$] Let $E\in\wadm(\BF')$. Then $T\not\subseteq E$, because $E$ is conflict-free and for every $t\in T$ there is some $T'\subseteq T$ with $(T',t)\in R$. Therefore $E$ is closed and conflict-free in $\BF$.It is left to show that $E$ weakly defends itself, \ie\, there is no attack $(T^*,h^*)$ with $h^*\in E$ and $T^*\setminus E\subseteq E^*$ for some $E^*\in\wadm(\BF'^E)$. If $E=\emptyset$ no such $h^*$ exists, so the statement holds. For $E\neq\emptyset$ the reduct $\BF^E$ contains strictly less arguments than $\BF$, so it satisfies the induction hypothesis. We can now distinguish the same cases as for the other direction $(\subseteq)$ and argue analogously.
        
        \end{description}
        
        \item[($\sigma=\wpref$)] Since weakly preferred extensions are the $\subseteq$ maximal weakly admissible extensions, this follows directly from the satisfaction of $\pPR$ by $\sigma=\wadm$.
        
        \item[($\sigma=\wcom$)] Let $(T,h)$ be a paradoxical attack. We show that $E\in\wadm(\BF)$ weakly defends some $X\supseteq E$ in $\BF$ iff $E$ weakly defends $X$ in $BF'$. We distinguish the following two cases:
        
        \begin{description}
        \item[(Case 1)] Suppose for all attacks $(T^*,h^*)$ on $X$ with $T^*\subseteq E^*$ for some closed $E^*$ it holds that $E$ attacks $E^*$ in $\BF$. Then, since $E$ is conflict-free and $T$ is not, none of the attacks from $E$ to $T^*$ coincide with $(T,h)$ so $E$ attacks $E^*$ in $\BF'$. Since we have more attacks in $BF$ than in $BF'$, the other direction is trivial for this case.
        
        \item[(Case 2)] Suppose there is at least one such $E^*$, which is not attacked by $E$ in $\BF'$, then it is also not attacked by $E$ in $\BF$ and vice versa. Furthermore, to all $E^{*'}$ containing other attacks on $X$ the reasoning of Case 1 still applies. It is left to show that 
        \begin{itemize}
        \item there is no $E^{**}\in\wadm(\BF^E)$ with $E^*\subseteq E\cup E^{**}$ iff there is no $E^{**}\in\wadm(\BF'^E)$ with $E^*\subseteq E\cup E^{**}$
        \item there is some $X'\in\wadm(\BF)$ such that $X\subseteq X'$ iff there is some $X'\in\wadm(\BF')$ such that $X\subseteq X'$
        \end{itemize}
        The second item follows directly from the satisfaction of $\pPR$ by weakly admissible semantics. For the first item, observe that in this case $(T\setminus E,h)$ is a paradoxical attack in $BF^E$, so again $\pPR$ for weak admissibility applies.
        \end{description}
        
        Now let $(T,h)$ be a paradoxical support. We show that $E\in\wadm(\BF)$ weakly defends some $X\supseteq E$ in $\BF$ iff $E$ weakly defends $X$ in $BF'$. We distinguish the following two cases:
                
                \begin{description}
                \item[(Case 1)] Suppose for all attacks $(T^*,h^*)$ on $X$ with $T^*\subseteq E^*$ for some closed $E^*$ it holds that $E$ attacks $E^*$ in $\BF$. Then, since $R=R'$, it holds that $E$ attacks $E^*$ in $\BF'$. We can use this argument for the other direction as well.
                
                \item[(Case 2)] Suppose there is at least one such $E^*$, which is not attacked by $E$ in $\BF'$, then it is also not attacked by $E$ in $\BF$ and vice versa. Furthermore, to all $E^{*'}$ containing other attacks on $X$ the reasoning of Case 1 still applies. It is left to show that 
                \begin{itemize}
                \item there is no $E^{**}\in\wadm(\BF^E)$ with $E^*\subseteq E\cup E^{**}$ iff there is no $E^{**}\in\wadm(\BF'^E)$ with $E^*\subseteq E\cup E^{**}$
                \item there is some $X'\in\wadm(\BF)$ such that $X\subseteq X'$ iff there is some $X'\in\wadm(\BF')$ such that $X\subseteq X'$
                \end{itemize}
                The second item follows directly from the satisfaction of $\pPR$ by weakly admissible semantics. For the first item, we distinguish the same cases as in the proof of $\pPR$ for $\adm^w$, for each of these cases we show $\wadm(\BF^E)=\wadm(\BF'^E)$, the first item follows directly from that.
                \end{description}
        
        \item[($\sigma=\wgrd$)] Since weakly grounded extensions are the $\subseteq$ minimal weakly complete extensions, this follows directly from the satisfaction of $\pPR$ by $\sigma=\wadm$.
        
        \end{description}
        
  \item[$\pM$] Follows directly from Prop.\ref{prop:modularization wadm}
        \item[$\pP$] Consider a BSAF $\BF$ containing an argument $a$, that attacks itself, \ie\, $\tuple{\{a\},a}\in R$, which is also supported by the empty set, $\tuple{\emptyset,a}\in S$. Then, $\wadm(\BF)=\emptyset\neq\{\emptyset\}=\wadm(\BF\downarrow_{A\setminus\{a\}})$.
        \item[$\pF$] The fundamental lemma is not satisfied: Consider an BSAF with arguements $a,b,c$, support $(\{a\},c)\in S$ and attack $(\{b\},c)\in R$. $a$ defends $b$ but $\{a,c\}$ is not closed, hence not adm.
        \item[$\pL$] The proof is analogous to the proof in~\cite{BlumelKU24}[Proposition~30]. Let $E$ be admissible in $F$. By definition, $E$ is conflict-free and closed, moreover, item 2 from Definition~\ref{def:WADMforABA} is satisfied since $T\cap E^+_R\neq \emptyset$ for each attacker $T$. Thus, $E$ is weakly admissible.
        \item[$\pNE$] is not satisfied: Consider an ABAF with two assumptions $a,b$, support $(\emptyset,a)$ and attack $(b,a)$. Then $\emptyset$ is not closed.
        \item[$\pMC$] Consider the BSAF used in the proof for the fundamental lemma above: $\BF=\tuple{\{a,b,c\},R,S}$, \suth $R=\{\tuple{b,c}\}$ and $S=\{\tuple{a,c}\}$. $\BF$ has two weakly admissible extensions, $\wadm(\BF)=\{\emptyset,\{b\}\}$, and therefore one weakly preferred extension, $\wpref(\BF)=\{\{b\}\}$. Yet $\BF$ has no complete extension.
        \item[$\pSR$] Consider the same counter-example as for $\pMC$, $\BF$ has a preferred, but no complete extensions.
    \end{itemize}    
\end{proof}

\section{On Weak and Strong $\Gamma$-Semantics (Omitted Proofs of Section~\ref{sec:FL-sensitive semantics})}
We discuss $\Gamma$-admissible semantics. 
Before delving into the proof details for strongly $\Gamma$-admissible semantics, 
we briefly discuss why the adaption of the reduct for weakly $\Gamma$-admissible 
semantics, despite sounding promising at first glance, does not recover the properties that are necessary to define a meaningful notion of weak admissibility. 

A possible adjustment of the reduct is to replace the original closure notion with $\Gamma$-closure.
\begin{definition}\label{def_bsaf_reduct_gamma}
	Given a BSAF $\BF=(A,R,S)$ and $E\subseteq A$, the \emph{$E$-$\Gamma$-reduct of $\BF$} is the BSAF $\BF^E=(A^E,R^E,S^E)$, with
	\begin{align*}
		A^E = &\ A \setminus (\cl_{\Gamma}(E)^\oplus_R)\\
		R^E = &\ \{(T\setminus \cl_{\Gamma}(E),t) \mid  \exists h\in \cl_{\Gamma}(E)^+_R: (T,h)\in S,\\
		& \phantom{\ \{(T\setminus \cl_{\Gamma}(E),t) \mid\quad}   t\in T\cap A^E\}\, \cup  \\
		&\ \{(T\setminus \cl_{\Gamma}(E), h) \mid  T\cap \cl_{\Gamma}(E)^+_R=\emptyset, \\
		& \phantom{\ \{(T\setminus \cl_{\Gamma}(E),t) \mid\quad} (T,h)\in R^E, h\in A^E
		\}\\
        S^E = 
		&\ \{(T\setminus \cl_{\Gamma}(E), h) \mid  T\cap \cl_{\Gamma}(E)^+_R=\emptyset, \\
		& \phantom{\ \{(T\setminus \cl_{\Gamma}(E),t) \mid\quad} (T,h)\in S, h\in A^E
		\}
	\end{align*}
\end{definition}
 As it turns out, however, the adaption does not satisfy modularization. 
We provide a counter-example for the $\Gamma$-reduct variant. 
\begin{example}\label{exm:bsaf-appFL}
Consider the following BSAF $\BF=\tuple{A,R,S}$ and its $\Gamma$-reduct wrt.\ $E=\{a,c\}$ below.
Note that $d$ is not in the $\Gamma$-closure of $E$ since it is not defended by $E$.
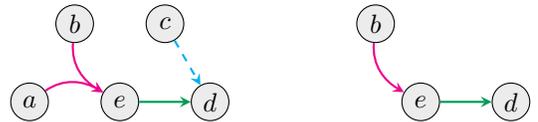
\vspace{-10pt}
\begin{center}
		\begin{tikzpicture}[scale=0.8,>=stealth]
		\path
        (\distance-1.1,.85)node (D){}
        (0,0)node[arg] (b){$b$}
		(\distance,0)node[arg] (c){$c$}
		(.5*\distance,-\distance*0.866)node[arg] (e){$e$}
		(-.5*\distance,-\distance*0.866)node[arg] (a){$a$}
  		(1.5*\distance,-\distance*0.866)node[arg] (d){$d$}
        ;
        
\path[thick,->,ForestGreen]
(e)edge(d)
;

\path[thick,->,cyan,dashed]
(c)edge(d)
;

\path[thick,->,magenta]
(a)edge[out=35,in=150](e)
(b)edge[out=-95,in=150](e)
;

\begin{scope}[xshift=5cm]
		\path
        (\distance-1.1,.85)node (D){}
        (0,0)node[arg] (b){$b$}
		(.5*\distance,-\distance*0.866)node[arg] (e){$e$}
  		(1.5*\distance,-\distance*0.866)node[arg] (d){$d$}
        ;
        
\path[thick,->,ForestGreen]
(e)edge(d)
;

\path[thick,->,magenta]
(b)edge[out=-95,in=150](e)
;

\end{scope}	
		\end{tikzpicture}
	\end{center}
$E$ is $\Gamma$-closed since $E$ does not defend $d$ although it is in the closure of $E$.  
Let $E'=\{b\}$.
$E$ is weakly $\Gamma$-admissible in $\BF$, $E'$ is weakly $\Gamma$-admissible in $\BF^E$.
However, $E\cup E'$ is not weakly $\Gamma$-admissible in $\BF$.
\end{example}
In the remaining part of this section, we discuss the ommitted proofs of Section~\ref{sec:FL-sensitive semantics}.

\begin{restatable}{lemma}{lemSTRGamma}
$\sgadm(\BF)\subseteq \gadm(\BF)$ for each $\BF$. 
\end{restatable}
\begin{proof}
Follows since $E$ strongly defends $a$ implies $E$ defends $a$, for each set of arguments $E$ and argument $a$.
\end{proof}

\PROPsgadmNonEmpty*
\begin{proof}
Follows from $\cl_{\Gamma}(\emptyset)$ being strongly admissible. One can argue for that analogously to the proof of Prop. \ref{prop:weak FL for sadm and sgadm} via iteration of $\Gamma$-closure, let $E=S_1=\emptyset$.
\end{proof}

\PROPstgadmDefinitions*
\begin{proof} The proof is analogous to the case of strong admissibility, cf. Prop.~\ref{prop:stadmDefinitions}, since \mbox{($\Gamma$-)closure} of a strongly \mbox{($\Gamma$-)}admissible set is only required in the final step and thus does not influence the iteration. 
\end{proof}

\PropWeakFundamentalLemmaStrongGamma*
\begin{proof}
    $E$ is strongly $\Gamma$-admissible, thus there is a sequence $E_1,\dots,E_n$ satisfying the conditions in the constructive definition in Proposition~\ref{def_stgadm_constructive}. We continue the sequence as follows:
    
    Let $\cl_\Gamma^p(E)=\cl(E)\cap \Gamma(E)$ denote a single iteration of $\Gamma$-support. 
    We define $S_1,\dots, S_m$ with 
    $S_1=\{a\}$ and
    $S_i=\cl^p_\Gamma(\bigcup_{j<i} S_{j}\cup E)\setminus (\bigcup_{j<i} S_{j}\cup E)$ for all $1<i\leq m$;
    and $\cl^p_\Gamma(\bigcup_{j\leq m} S_{j}\cup E)\setminus (\bigcup_{j\leq m} S_{j}\cup E)=\emptyset$.
    
    Let $E'=\bigcup_{j\leq m} S_{j}\cup E$.
    $E'$ is $\Gamma$-closed. 
    Moreover, for all $1\leq i,j\leq m$, 
    $S_i,S_j$ are pairwise disjoint and
    $S_i$ is defended by $\bigcup_{j<i}S_i\cup E$.
    Thus, the sequence $E_1,\dots,E_n, S_1,\dots, S_m$ for $E'$ satisfies the conditions in the constructive definition in Proposition~\ref{def_stgadm_constructive}.
\end{proof}

\begin{restatable}{proposition}{PROPsgadmProperties}
    \label{prop_stgadm_properties}
    Let $ \BF=(A,R,S)$ be a BSAF. Then,
    \begin{enumerate}
    \item If $E_1,E_2\in\sgadm(F)$, then there is $E\in \sgadm(F)$ s.t.\ $E_1\cup E_2\subseteq E$
    \item The $\subseteq$-maximal strongly $\Gamma$-$\adm$ set is unique $\pUM$ 
	\item Each admissible set has a unique biggest (wrt.\ set-inclusion) strongly $\Gamma$-admissible subset $\pURM$
    \end{enumerate}
\end{restatable}

\begin{proof}
    \hphantom{asd}
    \begin{enumerate}
    \item  $E_1$ and $E_2$ are conflict-free: 
    By definition, $E_i\in\sgadm(\BF)$ implies that there are disjoint sets $E^1_i,\dots,E_i^{n_i}$, \suth $E_i^1=\emptyset$, $E_i=\bigcup\limits_{j=1}^{n_i} E_i^j$ and for each $j\geq 1$ it holds that $E_i^j$ is defended by $\bigcup\limits_{k<j} E^k_i$. 
    Wlog suppose $S_1\subseteq E_2$ attacks $E_1$ on $a_1$. 
    Then $a_1\in E_1^i$ for some $i\leq n_1$ and $a_1$ is defended by some $T_1\subseteq \bigcup\limits_{k<i} E^k_1$, \ie, $T_1$ attacks some $b_1\in S_1$ and $b_1\in E_2^j$ for some $j\leq n_2$.
    $T_1$ (resp.\ some $a_2\in T_1$) is in turn attacked by some 
    $S_2\subseteq \bigcup\limits_{k<j} E^k_2$. Note that $a_2$ is contained in some $E_l$ with $l<i$.
    In this way, we construct two sequences of attacked arguments $a_1,\dots,a_{m_1}$ in $E_1$ and $b_1,\dots,b_{m_2}$ in $E_2$, $a_i\in E_{v_i}$, $b_i\in E_{u_i}$ so that $u_i<u_k$ whenever $i>k$, analogous for $b_i$.
    Since the sequences $E_i^j$ are finite,  $E_{m_1}=E_1^1 \subseteq \Gamma(\emptyset) $ and  $E_{m_2}=E_2^1 \subseteq \Gamma(\emptyset)$, contradiction to the assumption $E_1$ and $E_2$ are strongly $\Gamma$-admissible and thus conflict-free.
 
    We therefore can construct $E$ analogous to the proof for Proposition~\ref{prop:weak FL for sadm and sgadm}.
    \item Apply (1) iteratively to all strongly $\Gamma$-admissible sets.
    \item Let $E\in\adm(\D)$; apply (1) for all strongly $\Gamma$-admissible set in $E$. 
    \qedhere
    \end{enumerate}
\end{proof}

\propPropertiesStrongGammaAdm*
\begin{proof}\hphantom{asd}
    \begin{itemize}
        \item $\pSTR$: Let $E\in\sgadm(\BF)$. By Def. \ref{def_stadm_recursive} $E$ is conflict-free, $\Gamma$-closed and each member of $E$ is defended by a subset of $E$. Since $\Gamma$ is monotonic, $E$ defends all of its members, so $E\in\gadm(\BF)$.
        \item $\pNE$ is proven in Prop.~\ref{prop:sgadm non-empty}.
        \item $\pWF$ is proven in Prop.~\ref{prop:weak FL for sadm and sgadm}.
        \item $\pCC$ Let $E\in\sgadm(\BF)$ and $E'\in\gcom(\BF)$.
            By definition, $E\in\sgadm(\BF)$ implies that there are disjoint sets $E_1,\dots,E_n$, \suth $E_1=\emptyset$, $E=\bigcup\limits_{i=1}^n E_i$ and for each $i\geq 1$ it holds that $E_i$ is defended by $\bigcup\limits_{j<i} E_j$. Since $E_1=\emptyset$, $E'$ defends $E_1$. 
            
            Since $E'\in\gcom(\BF)$, we obtain $E_1\subseteq E'$. We repeat this argument for all $E_i$ with $1\leq i\leq n$, and obtain $E\subseteq E'$.
        \item $\pSR$: Given $\BF=\tuple{A,R,S}$, assume toward contradiction $E\in\sgpref(\BF)$, \suth $E\notin\sgcom(\BF)$. Then $E$ defends an argument $a\in A\setminus E$. If $E\cup\{a\}$ is $\Gamma$-closed, we found a strict super-set of $E$ that is strongly $\Gamma$-admissible, otherwise we iteratively add the $\Gamma$-closure and arguments that are defended. Since $\BF$ is finite, this process terminates. We receive $E'\supseteq E\cup\{a\}$, \suth $E'\in\sgadm(\BF)$. Contradiction!
        \item $\pMC$ follows from $\pSR$ by Prop.~\ref{prop:MCiffSR}.
        \item $\pUM$ is proven in Prop.~\ref{prop:properties_strong_g_adm}
        \item $\pURM$ is proven in Prop.~\ref{prop:properties_strong_g_adm}
        \item Example~\ref{exm:counter ex mod sgadm} is a counter-example to $\pM$.\qedhere
    \end{itemize}
\end{proof}

\end{document}